\documentclass[conference]{IEEEtran}
\IEEEoverridecommandlockouts
\usepackage{cite}
\usepackage{amsmath}
\usepackage{amsthm}

\usepackage{amsfonts}
\usepackage{amssymb}
\usepackage{booktabs}
\usepackage{enumerate}
\usepackage{color}
\usepackage{graphicx}
\usepackage{breakurl,moreverb,url}
\usepackage{verbatim}
\usepackage{todonotes}
\usepackage[colorlinks,bookmarksopen,bookmarksnumbered,citecolor=red,urlcolor=red]{hyperref}
\usepackage{mathtools} %for := \coloneqq, for parentheses
\DeclarePairedDelimiter\parentheses{\lparen}{\rparen}	%for () after an operator

\usepackage[ruled,vlined]{algorithm2e}
\newtheorem{theorem}{Theorem}

\newtheorem{lemma}{Lemma}
\newtheorem{corollary}{Corollary}

\theoremstyle{definition}

\usepackage[utf8]{inputenc}
\usepackage[T1]{fontenc}

\DeclareUnicodeCharacter{2212}{-}

\DeclareMathOperator{\bx}{{\bf x}}
\DeclareMathOperator{\by}{{\bf y}}
\DeclareMathOperator{\bz}{{\bf z}}
\DeclareMathOperator{\be}{{\bf e}}
\DeclareMathOperator{\bb}{{\bf b}}
\DeclareMathOperator{\bs}{{\bf s}}
\DeclareMathOperator{\bc}{{\bf c}}
\DeclareMathOperator{\bh}{{\bf h}}

\DeclareMathOperator{\bA}{{\bf A}}

\DeclareMathOperator{\tolsetAb}{\Sigma_{\forall\exists}\left([\bA], [\bb]\right)}
\DeclareMathOperator{\tolsetAx}{\Omega_{\forall\exists}\left([\bA], [\bx]\right)}
\DeclareMathOperator{\tolsetb}{\Sigma_{\forall\exists}\left(\bA, [\bb]\right)}
\DeclareMathOperator{\tolsetx}{\Omega_{\forall\exists}\left(\bA, [\bx]\right)}
\DeclareMathOperator{\tolsetmidAx}{\Omega_{\forall\exists}\left(mid([\bA]), [\bx]\right)}
\DeclareMathOperator{\tolsetmidAb}{\Sigma_{\forall\exists}\left(mid([\bA]), [\bb]\right)}
\DeclareMathOperator{\ba}{{\bf a}}
\DeclareMathOperator{\bq}{{\bf q}}
\DeclareMathOperator{\bv}{{\bf c}}
\DeclareMathOperator{\bJ}{{\bf J}}
\DeclareMathOperator{\bg}{{\bf g}}
\DeclareMathOperator{\bl}{{\bf l}}

\DeclareMathOperator{\image}{Im}
\DeclareMathOperator{\vertex}{vert}
\DeclareMathOperator{\conv}{conv}
\DeclareMathOperator{\diag}{diag}
\DeclareMathOperator{\MID}{mid}
\DeclareMathOperator{\RAD}{\Delta}
\DeclareMathOperator{\INF}{inf}
\DeclareMathOperator{\SUP}{sup}
\DeclareMathOperator{\MIN}{min}

\DeclareMathOperator{\MAG}{mag}
\DeclareMathOperator{\MIG}{mig}

\newcommand{\cmid}[1]{{\MID\parentheses*{#1}}}
\newcommand{\cabs}[1]{{|{#1}|}}
\newcommand{\crad}[1]{{\RAD{#1}}}
\newcommand{\cinf}[1]{{\INF{\left({#1}\right)}}}
\newcommand{\csup}[1]{{\SUP{\left({#1}\right)}}}

\newcommand{\cmag}[1]{{\MAG\parentheses*{#1}}}
\newcommand{\cmig}[1]{{\MIG{\left({#1}\right)}}}

\newcommand{\onum}[1]{\overline{{#1}}} 	%upper limit interval
\newcommand{\unum}[1]{\underline{{#1}}} 	%lower limit interval

\begin{document}

\title{Efficient Set-Based Approaches for the Reliable Computation of Robot Capabilities}

\author{Joshua K. Pickard,
        Vincent Padois,
        Milan Hlad\'ik,
        and David Daney,
\thanks{This work was supported by the Czech Science Foundation Grant P403-18-04735S.}% <-this % stops a space
\thanks{J. K. Pickard was with Auctus, Inria - IMS (Univ. Bordeaux / Bordeaux INP / CNRS UMR 5218), 33405 Talence, France. He is now with Eigen Innovations, 444 Aberdeen St $\sharp$255, Fredericton, New Brunswick, Canada (email: joshua.pickard@inria.fr).}% <-this % stops a space
\thanks{V. Padois, and D. Daney are with Auctus, Inria - IMS (Univ. Bordeaux / Bordeaux INP / CNRS UMR 5218), 33405 Talence, France (emails: vincent.padois@inria.fr and david.daney@inria.fr).}% <-this % stops a space
\thanks{M. Hlad\'ik is with Charles University, Faculty of Mathematics and Physics, Department of Applied Mathematics, Malostransk\'e n\'am. 25, 11800, Prague, Czech Republic (email: milan\_hladik@centrum.cz).}
% \thanks{Manuscript submitted on ... }
}

{}

\maketitle

\begin{abstract}
% To reliably model real robot characteristics, interval linear systems of equations allow to describe families of problems that consider sets of values.
% In this work, several classes of problems are considered. Reliable and efficient polytope, n-cube, and n-ball inner approximations are presented for each class.
% These inner approximations describe the common capabilities of a robotic manipulator corresponding to considered sets of values, which allows to easily account for typical complexities such as sets of joint states and design parameter uncertainties. \dd{reverse the paragraph: capabilities -> inner approximation -> polytope  }

To reliably model real robot characteristics, interval linear systems of equations allow to describe families of problems that consider sets of values.
This allows to easily account for typical complexities such as sets of joint states and design parameter uncertainties.
Inner approximations of the solutions to the interval linear systems can be used to describe the common capabilities of a robotic manipulator corresponding to the considered sets of values.
In this work, several classes of problems are considered. For each class, reliable and efficient polytope, n-cube, and n-ball inner approximations are presented.

The interval approaches usually proposed are inefficient because they are too computationally heavy for certain applications, such as control. We propose efficient new inner approximation theorems for the considered classes of problems. This allows for usage with real-time applications as well as rapid analysis of potential designs.
Several applications are presented for a redundant planar manipulator including locally evaluating the manipulator's velocity, acceleration, and static force capabilities, and evaluating its future acceleration capabilities over a given time horizon.
\end{abstract}


% \fi
\IEEEpeerreviewmaketitle

\section{Introduction}

\IEEEPARstart{T}{he} capability of a robot to autonomously adapt to an open dynamic environment is one of the ultimate challenges of robotics.

One of the keys to solve this challenge lies in the ability to dynamically control the motion of the robot as a function of the prediction of the environment dynamics in order to both optimally achieve the goals assigned to the robot while ensuring safety at all times. In other words, planning and control can no longer be considered as two separated problems solved sequentially. Indeed, this kind of approach works well only for simple robots evolving in static environments and solving basic tasks, where a large amount of time can be spent planning the motion offline and where the control problem consists of tracking, with almost no adaptation, the planned trajectory.

\begin{figure}[h!]
    \centering
    \includegraphics[width=0.5\textwidth]{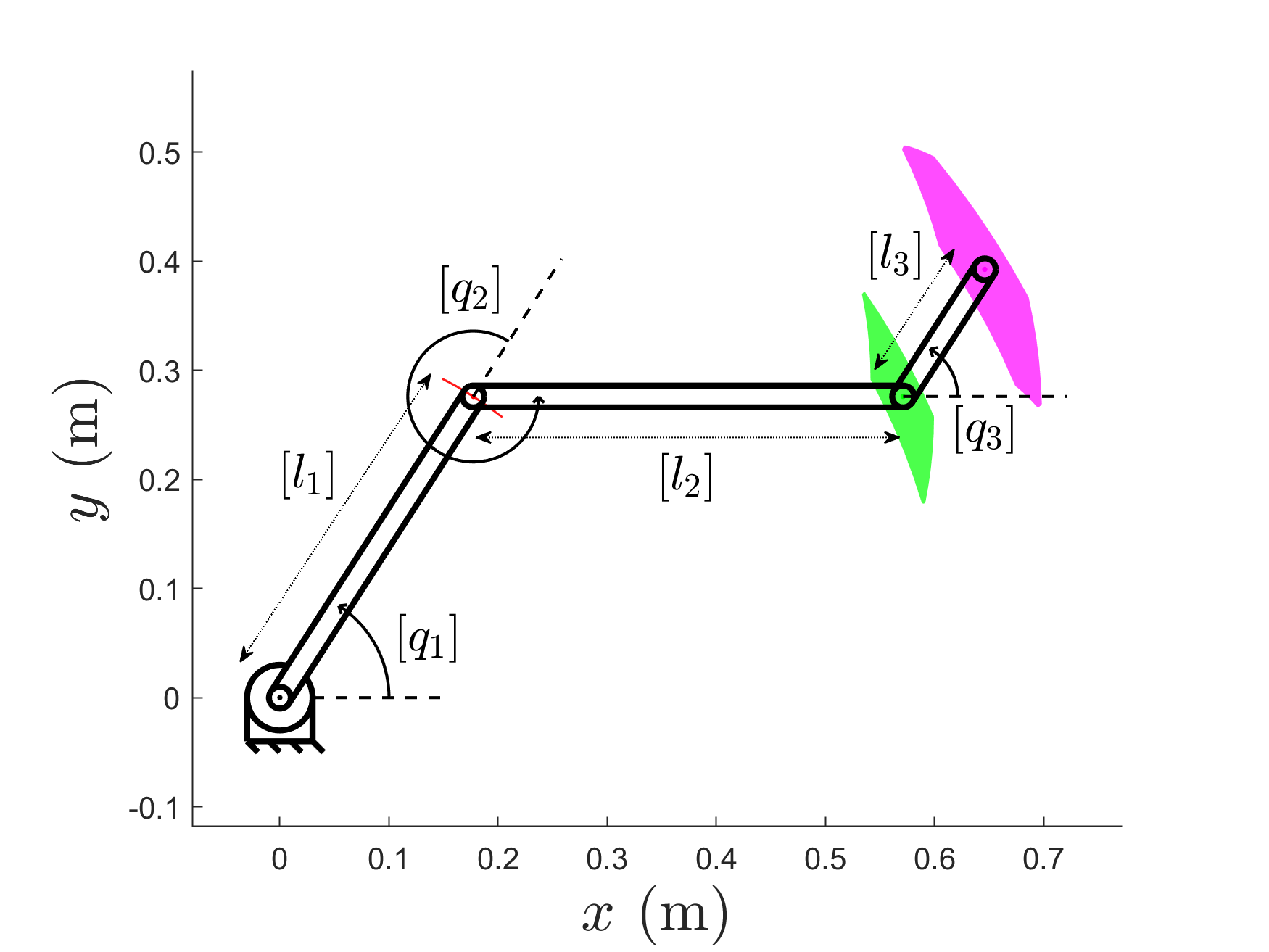}
    \caption{The sets of positions associated with the distal ends of the links for the set of configurations $[\bq] = (1.0,−1.0,1.0)\pm0.1$~rads for a family of three-link planar manipulators with parameter values $[\bl] = (0.3280,0.3940,0.1385)\pm 0.001$~m. The red, green, and purple regions are the corresponding positions of the distal ends of the first, second, and third links respectively.
    % \vp{Do we have / Can we have similar results with a 3D robot like the 7DOF Panda ? David was mentioning in an earlier discussion that given the simple nature of the robots chosen as examples, the reader may have the feeling that the proposed work only applies to simple systems.}
    }
    \label{fig:threelinkpositions}
\end{figure}

Beyond the perception abilities needed to estimate the state of the robot with respect to its environment, dynamic contexts require to adapt the planned motion at all time. For this adaptation to be optimal, it should not be purely reactive but rather adapt the overall plan online, i.e., at each control instant. While the concept of online re-planning is appealing, it raises the question of the real-time tractability of the corresponding optimization problems. As explained in \cite{Ibanez2017}, in the very challenging case of humanoid robots, accounting for limited computation power, one can choose to make use of rather precise models of the robot's dynamics and constraints but the price to pay for this modeling complexity is the limited time horizon over which planning can be considered, leading to feasible but potentially sub-optimal robot behaviours. On the other hand, one can decide to plan over longer time horizons, for example using optimal control approaches or their receding horizon version such as Model Predictive Control (MPC). In return, one has to make use of simpler models to maintain a low computation complexity for online re-planning to be possible. This is typically the approach taken when planning for locomotion in humanoid robotics: given a target position in Cartesian space, one first computes a desired trajectory for an abstracted version of the system, typically its centre of mass (CoM) and the associated centroidal dynamics. Then one defines a sequence of contacts and associated stances that are both compatible with the desired CoM trajectory and physically viable. Finally, using a more complete model of the system, a whole-body control approach is used to compute the actuation inputs allowing to follow the trajectory joining two consecutive stances and associated contact states \cite{tonneau2018efficient}. The latter approach seems more appropriate in terms of the optimality of the obtained behaviours but also in terms of the involved computations. Indeed, while recent works on Differential Dynamic Programming \cite{mastalli2019crocoddyl} tackle the full model problem over the full horizon, these approaches remain computationally too complex for online re-planning on real robots and it actually does not seem quite necessary to plan precisely joint level actuation over a long time horizon.

If approaches resorting to simplified models seem more appropriate, an abstract model of the system still requires to be connected to a more physically realistic one for control purposes. While this connection can be made through task level abstractions, it is also necessary to account for the low level limits of the robot at the planning level in order to produce plans that can be achieved by the robots given its intrinsic (e.g., joint limits) and extrinsic (e.g., collision avoidance) constraints. Often in robotics, accounting for the joint level limitations at the more abstract level of task description is a matter of manual tuning and heuristics as solving this problem formally and in a computationally efficient manner is complex. Going beyond simple heuristics and assuming constant task level acceleration capabilities (\textit{e.g.} the ones provided by robot manufacturers), some work have been performed in the past to account for acceleration limits at the joint level while accounting for the potential incompatibilities between these acceleration limits and position ones \cite{delprete2018joint,rubrecht2012motion,decre2009extending}. While of interest, these works fail short at accounting in a formal way for the true limits of robots which are best expressed at the joint level in terms of joint position, velocity, torque and torque derivative. Indeed, one of the main difficulties with joint level acceleration limits is that they are state dependant. This is all the more true when considering the expression of these limits in Cartesian space as the mappings from joint space to task space, being them at the static or dynamic level, are both state dependant and non-linear. In order to tackle this complexity, there is a need for new methodologies that are able to efficiently, accurately and with a limited computation cost predict the robot's capabilities in static but more importantly in dynamic scenarios. Furthermore, these methodologies should be able to handle imprecise/variable systems by accounting for system complexities such as sets of joint states and design parameter uncertainties to produce reliable results. Accounting for these uncertainties is of high importance when dealing with floating base systems such as humanoids or mobile robots in general~\cite{8723340}. Thus, modeling errors and state changes over a given time horizon should be properly accounted for when computing capabilities.

With this overall motivation in mind, this work presents several classes of problems that can be used to reliably and efficiently compute polytope, n-cube, and n-ball inner approximations of the capabilities of robotic systems, being them imprecise or uncertain.

The descriptions of the capabilities of many robotics systems can be formulated in terms of linear systems of equations
\begin{equation}
    \bA \bx = \bb
\end{equation}
with $\bA \in \mathbb{R}^{m \times n}$, $\bx \in \mathbb{R}^n$, and $\bb \in \mathbb{R}^m$.
These linear equations may be used to describe specific capabilities of a robotic manipulator for a unique configuration ${\bq}$ or state $\{{\bq},{\bf \dot{q}}\}$.
In offline applications, the evaluation of robot capabilities allows to analyze the performance of a given design and may be used during the dimensional synthesis stage of robotic design. In online applications, the current and/or future robot capabilities can be computed and may be used to autonomously re-plan and control the motion of the robot.

For example, for serial manipulators at the velocity kinematic level one may consider the:
\begin{itemize}
\item \textit{velocity kinematics model}:
\begin{equation}\label{eq:velocity_problem}
    \begin{aligned}
    &{\bJ}({\bq})\dot{\bq} = \dot{\bx}\\
    \end{aligned}
\end{equation}
where ${\bJ}({\bq})$ is the Jacobian matrix, and $\dot{\bq}$ and $\dot{\bx}$ are the joint and operational velocities respectively.
\item \textit{kinetostatic model}
\begin{equation}\label{eq:wrench_problem}
    \begin{aligned}
    &{\bJ}({\bq})^T{\bf f} = \boldsymbol{\tau}\\
    \end{aligned}
\end{equation}
where ${\bf f}$ is the operational wrench and $\boldsymbol{\tau}$ is the joint torque/force vector.
\end{itemize}

At the dynamic level, one may also characterize the capabilities of a robot through the:
\begin{itemize}
\item \textit{forward acceleration model}:
\begin{equation}\label{eq:acceleration_problem}
    \begin{aligned}
    &\dot{\bJ}({\bq},\dot{\bq}) \dot{\bq} + {\bJ}({\bq}) \ddot{\bq} = \ddot{\bx}\\
    \end{aligned}
\end{equation}
where $\dot{\bJ}({\bq},\dot{\bq})$ is the time derivative of the Jacobian matrix, and $\ddot{\bq}$ and $\ddot{\bx}$ are the joint and operational accelerations respectively.
\item \textit{configuration-space dynamic model}:
\begin{equation}\label{eq:configuration_dynamic_problem}
    \begin{aligned}
    &{\bf M}({\bq}) \ddot{\bq} = \boldsymbol{\tau} - {\bv}({\bq}, \dot{\bq}) - {\bg}({\bq})\\
    \end{aligned}
\end{equation}
where ${\bf M({\bq})}$ is the mass matrix, ${\bv}({\bq}, \dot{\bq})$ is the centrifugal and Coriolis vector and ${\bg}({\bq})$ is the gravity vector.
\item \textit{operational-space dynamic problem}:
\begin{equation}\label{eq:operational_dynamic_problem}
    \begin{aligned}
    &\boldsymbol{\Lambda}({\bq}) \ddot{\bx} = {\bf f} - \boldsymbol{\mu}({\bq}, \dot{\bq}) - {\bf p}({\bq})\\
    \end{aligned}
\end{equation}
where $\boldsymbol{\Lambda}({\bq})$ is the Cartesian mass matrix, $\boldsymbol{\mu}({\bq}, \dot{\bq})$ is the centrifugal and Coriolis vector in Cartesian space and ${\bf p}({\bq})$ is the gravity vector in Cartesian space.
% \vp{Here we may want to use Khatib's notation which are commonly used in the control community at least}
\end{itemize}

Due to the many sources of uncertainties that exist (\emph{e.g.}, noisy sensor data, modeling errors), to reliably predict the capabilities of the robot it can be desirable to consider sets of robot models and states.
Using interval analysis (see~\cite{Moore:107064,jaulin:hal-00845131, MooKea2009}), it is possible to extend these problems in order to describe a family of problems that consider sets of values (\emph{e.g.}, sets of design parameters, sets of configurations $[{\bq}]$, sets of joint velocities $[\dot{\bq}]$, etc.).
Figure~\ref{fig:threelinkpositions} depicts a family of three-link planar manipulators with sets of link lengths, $[l_1]$, $[l_2]$, and $[l_3]$, and sets of configurations, $[q_1]$, $[q_2]$, and $[q_3]$, and shows the corresponding sets of positions associated with the distal ends of the links for the given configurations~\eqref{eq:config_eg} and design parameters~\eqref{eq:parametervalues}. The set of configurations may be used to model dimensional uncertainties in the system, and may also be used to model the temporal evolution of the system. The set-based approaches presented in this work can be used to reliably predict the robot’s various capabilities for these type of problems, where the reliable and efficient evaluation of imprecise/variable systems  and their current or future capabilities over a given prediction period provides many avenues for improving the performance and safety of robotic manipulators.
Classes of problems may be described in terms of quantifiers which allow to describe reliable inner approximations of the capabilities of a robot over sets of values.
Alternatively, reliable outer approximations may also be computed~\cite{PopHla2013,Shary2001,Kol2004c,PopKra2007}, although these types of problems are not considered in this work.
Considering velocity kinematics related models, when a set of configurations $[{\bq}]$ are considered the Jacobian matrix ${\bJ}$ becomes an interval matrix $[{\bJ}]$. Generally, the evaluation of $[{\bJ}]$ is overestimated due to the interval wrapping effect~\cite{Neumaier1993} and dependency problem~\cite{Nedialkov2004} and there exists some matrices ${\bJ}_o \in [{\bJ}]$ such that ${\bJ}_o \not\in \{{\bJ}({\bq}) ~|~ {\bq} \in [{\bq}] \}$. Nevertheless, each ${\bJ} \in [{\bJ}]$ has associated capabilities and the intersection of these capabilities describes the common capabilities for all configurations ${\bq} \in [{\bq}]$. This extends to the dynamic models similarly.

The interval analysis framework has proven to be a useful tool for the analysis and synthesis of robotic mechanisms as it guarantees certified numerical solutions to problems~\cite{merlet:inria-00362431}. When combined with classical branch-and-bound approaches, entire parameter spaces can be thoroughly and reliably explored allowing the worst-case conditions throughout the parameter space to be properly understood. Such an approach has been widely used for workspace analysis, \emph{e.g.}, considering flexure-jointed mechanisms~\cite{10.1115/1.3042151}, evaluating wrench capabilities~\cite{5657268,doi:10.1139/tcsme-2016-0012} and positioning errors~\cite{Merlet2008}, and computing collision-free poses~\cite{farzanehkaloorazi_masouleh_caro_2017} and singularity-free poses~\cite{10.1007/978-3-642-14743-2_13,10.1007/978-1-4020-4941-5_5}.

By considering the intersections of manipulator capabilities for all values in the sets, this allows to compute common capabilities of the robotic manipulator corresponding to the considered sets of values.
To reliably describe the capabilities of a robotic manipulator corresponding to a set of values, inner approximations of the intersections of manipulator capabilities must be computed.
For a given interval linear system of equations, several classes of problems may be considered.
The main contribution of this work is the development of efficient polytope, n-cube, and n-ball inner approximations which are applicable to each class of problem considered.
Furthermore, these inner approximations are applied to several relevant robotics problems to demonstrate their capability to easily handle typical complexities such as sets of joint states and design parameter uncertainties.

The remainder of this paper is as follows.
First, in Section~\ref{sec:notations} the relevant interval notations are introduced.
Next, in Section~\ref{sec:class} the classes of interval linear system of equation problems considered in this work are described. In Section~\ref{sec:innerapprox} the polytope, n-cube, and n-ball inner approximations for each class of problem are presented. Several applications are presented in Section~\ref{sec:applications} for a redundant planar manipulator to locally evaluate the manipulator's velocity, static force, and acceleration capabilities, and also its future acceleration capabilities over a given time horizon.

\section{Notation}\label{sec:notations}
The interval extension of a variable $x$ is given by
\begin{equation}
    [x] = \{x~|~\unum{x} \le x \le \onum{x}\}
\end{equation}
where $\unum{x}$ is the lower bound (infimum) $\unum{x} = \cinf{x}$ and $\onum{x}$ is the upper bound (supremum) $\onum{x} = \csup{x}$ of the interval.

A few terms useful in this work are:
\begin{enumerate}
    \item The \textit{midpoint} of an interval is the real value given by
    \begin{equation}
        \cmid{[x]} = \frac{1}{2}(\unum{x} + \onum{x})
    \end{equation}

    \item The \textit{radius} of an interval is the real value given by
    \begin{equation}
        \crad[x] = \frac{1}{2}(\onum{x} - \unum{x})
    \end{equation}

    \item The \textit{mignitude} of an interval is the real value given by
    \begin{equation}
        \cmig{[x]} = \min(\{\cabs{x}~|~ x \in [x]\})
    \end{equation}

    \item The \textit{magnitude} of an interval is the real value given by
    \begin{equation}
        \cmag{[x]} = \max(\{\cabs{x}~|~ x \in [x]\})
    \end{equation}

    \item The \textit{absolute value} of an interval is the interval given by
    \begin{equation}
        \cabs{[x]} = [\cmig{[x]},\cmag{[x]}]
    \end{equation}
\end{enumerate}

Let a vector and matrix be given by $\bx$ and $\bA$ respectively. The interval extensions of the vector and matrix are then given by $[\bx]$ and $[\bA]$ and the previously described terms may be applied to each element in the vector or matrix.
Classes of problems may be described with the $\forall$ (for all) and $\exists$ (there exists) quantifiers which allow to describe reliable inner approximations to the interval linear system of equations over the considered sets of values.

\section{Classes of Problems}\label{sec:class}
Several classes of problems are considered in this work which are presented in a general form $\bA \bx = \bb$. These problems and their corresponding sub-problems are as follows.

\subsection{Problems $\tolsetAx$ and $\tolsetx$}
The first class of problems considered are of the form
\begin{equation}\label{eq:linear_system_b}
    \tolsetAx = \{ \bb ~|~(\forall \bA \in [\bA])(\exists \bx \in [\bx])(\bA \bx = \bb)\}
\end{equation}
which applies a set of linear transformations $[\bA]$ to a set of vectors $[\bx]$.
% It is well-known in the robotics literature \vp{Proof by intimidation. Also this result goes further than just Robotics, doesn't it?} that
The sub-problem
\begin{equation}
    \tolsetx = \{ \bb ~|~(\exists \bx \in [\bx])(\bA \bx = \bb)\}
\end{equation}
is geometrically a convex polytope~\cite{Bouchard2009,carretero-gosselin-2010-cable,Gouttefarde2010475482}. Therefore, \eqref{eq:linear_system_b} is equivalent to the intersection of all sub-problems $\bA \in [\bA]$. That is,
\begin{equation}\label{eq:linear_system_b_inter}
    \tolsetAx = \bigcap_{\bA \in [\bA]}\tolsetx
\end{equation}

\subsection{Problems $\tolsetAb$ and $\tolsetb$}
The second class of problems considered are of the form
\begin{equation}\label{eq:linear_system_x}
    \tolsetAb = \{\bx~|~(\forall \bA \in [\bA])(\exists~\bb \in [\bb])(\bA \bx = \bb)\}
\end{equation}
which is precisely the tolerance solution set of the interval linear system as described in~\cite{Shary2002}.
Considering the sub-problem
\begin{equation}
    \tolsetb = \{\bx~|~(\exists~\bb \in [\bb])(\bA \bx = \bb)\}
\end{equation}
which is also geometrically a convex polytope, it is clear that
\begin{equation}\label{eq:linear_system_x_inter}
\begin{aligned}
\tolsetAb = \bigcap_{\bA \in [\bA]} \tolsetb
\end{aligned}
\end{equation}

\subsection{Robotics Use Cases}
The use cases of each class of problem applied to the relevant robotics problems~\eqref{eq:velocity_problem} -- %, \eqref{eq:wrench_problem}, \eqref{eq:acceleration_problem}, \eqref{eq:configuration_dynamic_problem}, and
\eqref{eq:operational_dynamic_problem} are outlined in Tables~\ref{tab:usecases_tolsetx} %, \ref{tab:usecases_tolsetAx}, \ref{tab:usecases_tolsetb}, and \ref{tab:usecases_tolsetAb}.
and \ref{tab:usecases_tolsetb}.
The classes $\tolsetx$ and $\tolsetb$ consider a unique configuration $\bq$ or state $\{\bq, \dot{\bq}\}$ whereas the classes $\tolsetAx$ and $\tolsetAb$ can consider a set of configuration $[\bq]$ or states $\{[\bq], [\dot{\bq}]\}$.

\bgroup
\def\arraystretch{1.5}%  1 is the default, change whatever you need
\begin{table*}[t]
\setlength{\fboxsep}{0pt}%
\setlength{\fboxrule}{0pt}%
\begin{center}
\caption{Inner approximation use cases for $\tolsetx$ and $\tolsetAx$ \label{tab:usecases_tolsetx}}
\begin{tabular}{|l|c|c|c|}
     \hline
     {\bf Problems} & Unknown & Description of $\tolsetx$ & Description of $\tolsetAx$   \\\hline
     ${\bJ}({\bq})\dot{\bq} = \dot{\bx}$ & $\dot{\bx}$ &
     $\{\dot{\bx}~|~(\exists \dot{\bq} \in [\dot{\bq}])({\bJ}({\bq})\dot{\bq} = \dot{\bx})\}$ &
     $\{\dot{\bx}~|~(\forall {\bJ}({\bq}) \in {\bJ}([{\bq}]))(\exists \dot{\bq} \in [\dot{\bq}])({\bJ}({\bq})\dot{\bq} = \dot{\bx})\}$ \\\hline

     ${\bJ}({\bq})^T{\bf f} = \boldsymbol{\tau}$ & $\boldsymbol{\tau}$ &
     $\{\boldsymbol{\tau}~|~(\exists {\bf f} \in [{\bf f}])({\bJ}({\bq})^T{\bf f} = \boldsymbol{\tau})\}$ &
     $\{\boldsymbol{\tau}~|~(\forall {\bJ}({\bq}) \in {\bJ}([{\bq}]))(\exists {\bf f} \in [{\bf f}])({\bJ}({\bq})^T{\bf f} = \boldsymbol{\tau})\}$ \\\hline

     $\dot{\bJ}({\bq},\dot{\bq})\dot{\bq} + {\bJ}({\bq})\ddot{\bq} = \ddot{\bx}$ & $\ddot{\bx}$ &
     $\{\ddot{\bx}~|~(\exists \dot{\bq} \in [\dot{\bq}])(\dot{\bJ}({\bq},\dot{\bq})\dot{\bq} = \ddot{\bx})\}$ &
     $\{\ddot{\bx}~|~(\forall \dot{\bJ}({\bq},\dot{\bq}) \in \dot{\bJ}([{\bq}],[\dot{\bq}]))(\exists \dot{\bq} \in [\dot{\bq}])(\dot{\bJ}({\bq},\dot{\bq})\dot{\bq} = \ddot{\bx})\}$ \\
     & & $+ \{\ddot{\bx}~|~(\exists \ddot{\bq} \in [\ddot{\bq}])({\bJ}({\bq})\ddot{\bq} = \ddot{\bx})\}$ &
     $+ \{\ddot{\bx}~|~(\forall {\bJ}({\bq}) \in {\bJ}([{\bq}]))(\exists \ddot{\bq} \in [\ddot{\bq}])({\bJ}({\bq})\ddot{\bq} = \ddot{\bx})\}$
     \\\hline
\end{tabular}
\end{center}
\end{table*}
\egroup

\bgroup
\def\arraystretch{1.5}%  1 is the default, change whatever you need
\begin{table*}[t]
\setlength{\fboxsep}{0pt}%
\setlength{\fboxrule}{0pt}%
\begin{center}
\caption{Inner approximation use cases for $\tolsetb$ and $\tolsetAb$.\label{tab:usecases_tolsetb}}
\begin{tabular}{|l|c|c|c|}
     \hline
     {\bf Problems} & Unknown & Description of $\tolsetb$ & Description of $\tolsetAb$ \\\hline
     ${\bJ}({\bq})\dot{\bq} = \dot{\bx}$ & $\dot{\bq}$ &
     $\{\dot{\bq}~|~(\exists \dot{\bx} \in [\dot{\bx}])({\bJ}({\bq})\dot{\bq} = \dot{\bx})\}$ &
     $\{\dot{\bq}~|~(\forall {\bJ}({\bq}) \in {\bJ}([{\bq}]))(\exists \dot{\bx} \in [\dot{\bx}])({\bJ}({\bq})\dot{\bq} = \dot{\bx})\}$\\\hline

     ${\bJ}({\bq})^T{\bf f} = \boldsymbol{\tau}$ & ${\bf f}$ &
     $\{{\bf f}~|~(\exists \boldsymbol{\tau} \in [\boldsymbol{\tau}])({\bJ}({\bq})^T{\bf f} = \boldsymbol{\tau})\}$&
     $\{{\bf f}~|~(\forall {\bJ}({\bq}) \in {\bJ}([{\bq}]))(\exists \boldsymbol{\tau} \in [\boldsymbol{\tau}])({\bJ}({\bq})^T{\bf f} = \boldsymbol{\tau})\}$\\\hline

     ${\bf M}({\bq}) \ddot{\bq} = \boldsymbol{\tau} - {\bv}({\bq}, \dot{\bq}) - {\bg}({\bq})$ & $\ddot{\bq}$ &
     $\{\ddot{\bq}~|~(\exists \boldsymbol{\tau} \in [\boldsymbol{\tau}])(\exists \bq \in [\bq])(\exists \dot{\bq} \in [\dot{\bq}])$&
     $\{\ddot{\bq}~|~(\forall {\bf M}({\bq}) \in {\bf M}([{\bq}]))(\exists \boldsymbol{\tau} \in [\boldsymbol{\tau}])(\exists \bq \in [\bq])(\exists \dot{\bq} \in [\dot{\bq}])$\\
     &&$({\bf M}({\bq}) \ddot{\bq} = \boldsymbol{\tau} - {\bv}({\bq}, \dot{\bq}) - {\bg}({\bq}))\}$&
     $({\bf M}({\bq}) \ddot{\bq} = \boldsymbol{\tau} - {\bv}({\bq}, \dot{\bq}) - {\bg}({\bq}))\}$\\\hline

     $\boldsymbol{\Lambda}({\bq}) \ddot{\bx} = {\bf f} - \boldsymbol{\mu}({\bq}, \dot{\bq}) - {\bf p}({\bq})$ & $\ddot{\bx}$ &
     $\{\ddot{\bx}~|~(\exists {\bf f} \in [{\bf f}])(\exists \bq \in [\bq])(\exists \dot{\bq} \in [\dot{\bq}])$&
      $\{\ddot{\bx}~|~(\forall \boldsymbol{\Lambda}({\bq}) \in \boldsymbol{\Lambda}([{\bq}]))(\exists {\bf f} \in [{\bf f}])(\exists \bq \in [\bq])(\exists \dot{\bq} \in [\dot{\bq}])$\\
      &&$(\boldsymbol{\Lambda}({\bq}) \ddot{\bx} = {\bf f} - \boldsymbol{\mu}({\bq}, \dot{\bq}) - {\bf p}({\bq}))\}$&
     $(\boldsymbol{\Lambda}({\bq}) \ddot{\bx} = {\bf f} - \boldsymbol{\mu}({\bq}, \dot{\bq}) - {\bf p}({\bq}))\}$\\\hline
\end{tabular}
\end{center}
\end{table*}
\egroup
%%%%%%%%%%%%%%%%%%%%%%%%%%%%%%%%%%%%%%%%%%%%%%%%%%%%%%%

\section{Inner Approximations}\label{sec:innerapprox}
\subsection{$\tolsetx$ Inner Approximations}\label{sec:tolsetx}

Geometrically, polytope $\tolsetx$ represents a so called zonotope \cite{Zie2004}.
% Zonotopes are images of hypercubes under linear mappings and possess a special kind of symmetry, which can be utilized when solving various problems on zonotopes. \dd{un peu trop de zonotope}
% When $n\ge m$ {\mh[In the opposite case not?]} (\emph{e.g.}, for redundant and non-redundant serial manipulators)
The $\tolsetx$ polytope can be obtained as the convex hull ($\conv$) of the set of $\bb$ associated with the $2^n$ vertices ($\vertex$) of $[\bx]$. That is, the $\tolsetx$ polytope can be computed by
\begin{equation}\label{eq:tolsetx}
\tolsetx = \conv(\{\bb~|~ (\exists \bx \in \vertex([\bx]))(\bA \bx = \bb)\})
\end{equation}
An equivalent non-iterative approach called the hyperplane shifting method, proposed in~\cite{Bouchard2009} and further improved in~\cite{Gouttefarde2010475482}, may be used as a fast alternative in place of the convex hull routine. It requires to consider
% $\frac{n!}{(n-(m-1))!(m-1)!}$
sub-matrices formed from all of the possible combinations of $m-1$ columns of $\bA$ and directly computes the corresponding set of shifted hyperplanes.

Inner approximations of the largest inscribed n-ball and n-cube can be computed from the half-space representation of the $\tolsetx$ polytope. Let the half-space representation of the $\tolsetx$, describing the associated set of $\bb$, be given by
\begin{equation} \label{eq:tolsetx_HP}
    {\bf H}\bb \le {\bf d}
\end{equation}
with ${\bf H} \in \mathbb{R}^{k \times m}$ and ${\bf d} \in \mathbb{R}^k$. The hyperplane shifting method provides an efficient means of evaluating the half-space representation, especially when $m$ is close to $n$ as is common for many robotics problems.

\begin{theorem}[Largest inscribed n-cube centered at ${\bb_c}$]\label{th:r_cube_tolsetx}
If ${\bb_c}$ belongs to $\tolsetx$, then the n-cube
\begin{equation}
\begin{aligned}[]
[\bc] = {\bb_c} + r[\be]
\end{aligned}
\end{equation}
where $[\be]$ is a vector of $[-1,1]$ intervals and
\begin{equation} \label{eq:tolsetx_cube_r}
\begin{aligned}
r=\min_{i=1,\ldots,k} \frac{-\bh_i {\bb_c} + d_i}{\|\bh_i\|_1}
\end{aligned}
\end{equation}
 where $\bh_i$ is the $i^{th}$ row of ${\bf H}$ and ${\|\cdot\|_1}$ is the 1-norm, is contained in $\tolsetx$.
\end{theorem}
\begin{proof}
The Chebyshev center ${\bb_c}$ of a polytope yielding the largest inscribed n-cube can be determined by solving the linear program
\begin{equation}\label{eq:tolsetx_chebyshev_center_cube}
    \begin{aligned}
        &\max              & &r\\
        &\text{subject to} & &\bh_i {\bb_c} + r \|\bh_i\|_1 \le d_i\\
        & & & r \ge 0\\
        & & & i=1,\ldots,k\\
    \end{aligned}
\end{equation}
The maximum value for $r$ may be computed directly for a given ${\bb_c}$ as~\eqref{eq:tolsetx_cube_r}.
\end{proof}

Many extensions have been considered for Theorem~\ref{th:r_cube_tolsetx}, especially on the topic of tolerance sensitivity analysis (see~\cite{WenChen2010,Hla2011c}).

\begin{theorem}[Largest inscribed n-ball centered at ${\bb_c}$]\label{th:r_ball_tolsetx}
If ${\bb_c}$ belongs to $\tolsetx$, then the n-ball
\begin{equation}
\begin{aligned}
\mathcal{B} = \{ \bb ~ | ~ \|\bb - {\bb_c}\|_2 \le r\}
\end{aligned}
\end{equation}
with
\begin{equation} \label{eq:tolsetx_ball_r}
\begin{aligned}
r=\min_{i=1,\ldots,k} \frac{-\bh_i {\bb_c} + d_i}{\|\bh_i\|_2}
\end{aligned}
\end{equation}
 where ${\|\cdot\|_2}$ is the 2-norm, is contained in $\tolsetx$.
\end{theorem}
% The proof for Theorem~\ref{th:r_ball_tolsetx} follows directly from Theorem~\ref{th:r_cube_tolsetx}.
\begin{proof}
The Chebyshev center ${\bb_c}$ of a polytope yielding the largest inscribed n-ball can be determined by solving the linear program
\begin{equation}\label{eq:tolsetx_chebyshev_center_ball}
    \begin{aligned}
        &\max              & &r\\
        &\text{subject to} & &\bh_i {\bb_c} + r \|\bh_i\|_2 \le d_i\\
        & & & r \ge 0\\
        & & & i=1,\ldots,k\\
    \end{aligned}
\end{equation}
The maximum value for $r$ may be computed directly for a given ${\bb_c}$ as~\eqref{eq:tolsetx_ball_r}.
\end{proof}

As an example, consider the $\tolsetx$ problem given
\begin{equation}\label{eq:example_values}
    \bA =
    \begin{pmatrix}
    0.8947  & 0.6707  & 0.2409\\
    0.3348  & 0.3899  & 0.6958\\
    \end{pmatrix},
    [\bx] =
    \begin{pmatrix}
    [-1,1]\\
    [-1,1]\\
    [-1,1]\\
    \end{pmatrix}.
\end{equation}
The following inner approximations of $\tolsetx$ are given in Figure~\ref{fig:tolsetx}:
\begin{itemize}
    \item the associated polytope (purple); % \vp{My eyes are getting old: this blue looks purple to me.};
    \item the largest n-cube centered at the origin with $r=0.4686$ (red);
    \item the largest n-ball centered at the origin with $r=0.6407$ (green);
    \item the largest n-cube with variable center with $b_c = (0.4855,0.2823)$ and $r=0.4686$ (dashed line);
    \item the largest n-ball with variable center with $b_c = (0.2921,0.1698)$ and $r=0.6407$ (dashed line).
\end{itemize}
When reliable arithmetic operations are used which account for round-off errors, the computed inner approximations provide a reliable description of the solution of the linear transformation.
In this example there is a diagonal line of optimal solutions corresponding to the parallel edges that the variable center inner approximation algorithm can return.

\begin{figure}[h!]
    \centering
    \includegraphics[width=0.4\textwidth,trim={0cm 0.0cm 1cm 0.0cm},clip]{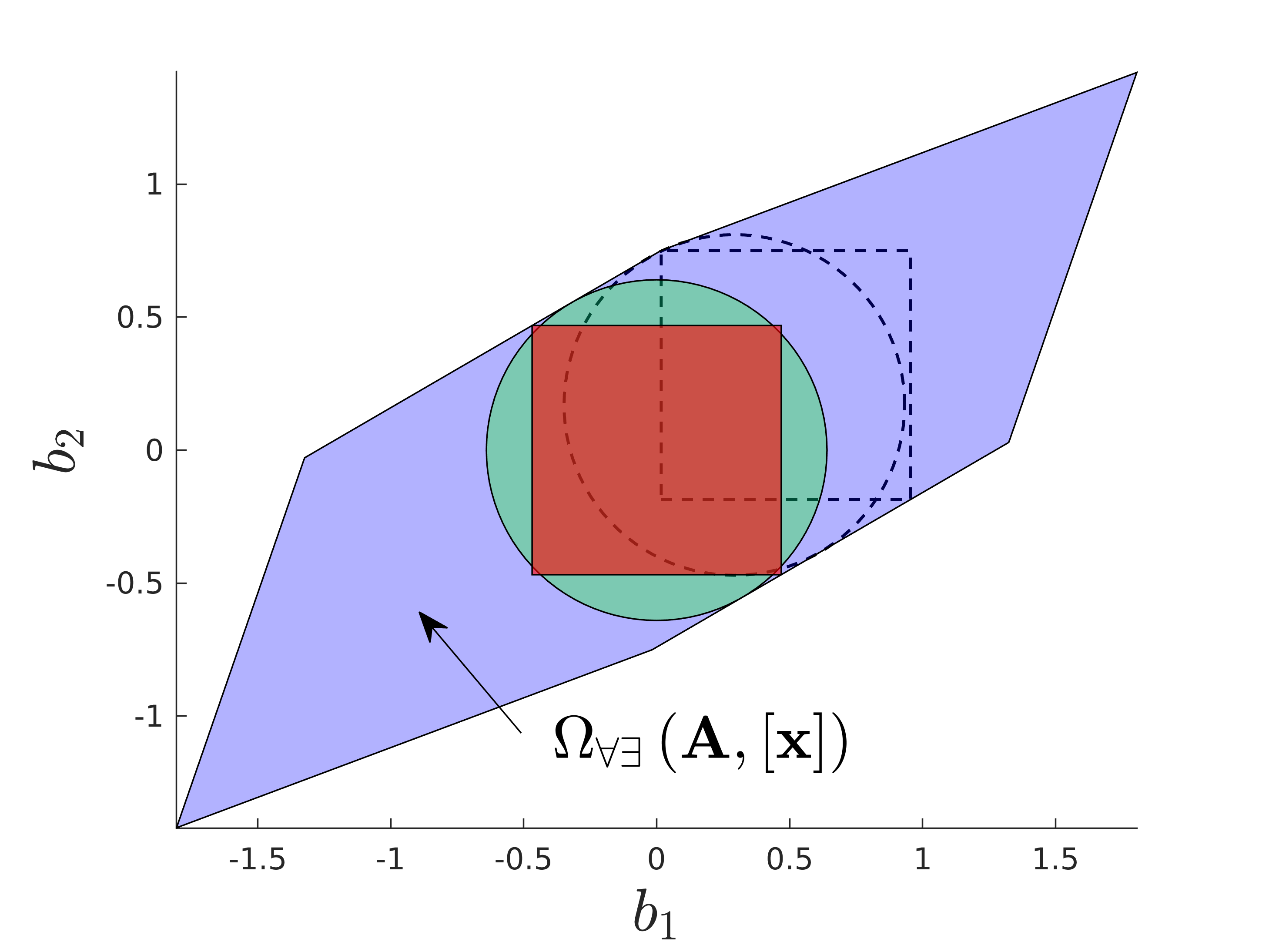}
    \caption{The associated polytope, n-cube, and n-ball inner approximations of $\tolsetx$.}
    \label{fig:tolsetx}
\end{figure}

\subsection{$\tolsetAx$ Inner Approximations}\label{sec:tolsetAx}
As stated in~\eqref{eq:linear_system_b_inter}, the common capabilities are the intersection of the capabilities for all sub-problems $\bA \in [\bA]$.
Since the capabilities associated with each sub-problem are geometrically convex polytopes, the common capabilities are the intersection of the convex polytopes for all sub-problems.

\begin{theorem}
We have $\bb\in\tolsetAx$ if and only if the linear system
\begin{align*}
&\bb=(\cmid{[\bA]}-\diag(\bs)\RAD[\bA])\bx_1\\
&~~~~~~~~-(\cmid{[\bA]}~+\diag(\bs)\RAD[\bA])\bx_2,\\
&\bx_1,\bx_2\geq0,\\
&\unum{\bx}\leq\bx_1-\bx_2\leq\onum{\bx}
\end{align*}
is solvable for each $\bs\in\{\pm1\}^m$, where $\diag(\bs)$ generates a diagonal matrix from the vector $\bs$.
\end{theorem}

\begin{proof}
For a fixed $\bb$, we have $\bb\in\tolsetAx$ if and only if the linear system
\begin{align}\label{sysPfThmTolSetAxChar}
\bb=\bA\bx,\ \ \unum{\bx}\leq\bx\leq\onum{\bx}
\end{align}
is solvable in variable $\bx$ for each $\bA\in[\bA]$. By~\cite{Hla2013b}, this kind of strong solvability is equivalent to solvability of the system from the theorem statement.
\end{proof}

The above test requires to solve $2^m$ linear systems. This can be a large number, however, it can hardly be decreased in general since the problem is intractable.

\begin{theorem}
The test $\bb\in\tolsetAx$ is co-NP-hard.
\end{theorem}

\begin{proof}
Suppose that $[\bx]$ is sufficiently large. Then the test $\bb\in\tolsetAx$ is equivalent to stating that for each $\bA\in[\bA]$ there is $\bx$ such that $\bA\bx=\bb$. That is equivalent with claiming that the interval linear system $[\bA]\bx=\bb$ is strongly solvable. Strong solvability is, however, known to be co-NP-hard~\cite{Roh1998b}. The real valued right-hand side $\bb$ of the linear system can be easily deduced.
% \vp{Not sure which right-hand side you are referring to}.
\end{proof}

To make the problem tractable, we first construct an inner approximation of $\tolsetAx$ by a zonotope, that is, by a set of the form $\Omega_{\forall\exists}\left(\bA, [\by]\right)$ for some $\bA$ and $[\by]$. In particular, we will choose $\bA:=\cmid{[\bA]}$ and  $[\by]:=\cmid{[\bx]}+[-r,r]\RAD[\bx]$. The task is now to compute as large as possible value $r\in[0,1]$ such that $\Omega_{\forall\exists}\left(\cmid{[\bA]}, [\by]\right)\subseteq \tolsetAx$. Notice that this need not be satisfied even for $r=0$, so we present a sufficient condition below.

\begin{theorem}\label{th:r_polytope}
Suppose that $\cmid{[\bA]}$ has full row rank and put
\begin{equation}\label{eq:tolsetAx_r_polytope}
\begin{aligned}
&r=\min_{i=1,\dots,n}\left(
  \frac{\left(\RAD[\bx]-|\cmid{[\bA]}^\dag|\cdot\RAD[\bA]
    \cdot\cmag{[\bx]}\right)_i}
  {\RAD[\bx]_i} \right),
\end{aligned}
\end{equation}
% \begin{equation}\label{eq:tolsetAx_r_polytope}
% \begin{aligned}
% &r=\min_{i=1,\dots,n}\left(
%   \frac{\left(\RAD[\bx]-|\cmid{[\bA]}^\dag|\cdot\RAD[\bA]
%     \cdot{\jp\RAD[\bx]}\right)_i}
%   {\RAD[\bx]_i} \right),
% \end{aligned}
% \end{equation}
where $(\cdot)^{\dag}$ denotes the Moore--Penrose pseudoinverse. If $r\geq0$, then $\Omega_{\forall\exists}\left(\cmid{[\bA]}, [\by]\right)\subseteq \tolsetAx$.
\end{theorem}

\begin{proof}
We want  $\Omega_{\forall\exists}\left(\cmid{[\bA]}, [\by]\right)\subseteq \tolsetx$ for each $\bA\in[\bA]$. That is, for each $\by\in[\by]$ there must be $\bx\in[\bx]$ such that $\cmid{[\bA]}\by=\bA \bx$. Putting $\bA':=\cmid{[\bA]}-\bA$, we can write it as
$$
\cmid{[\bA]}\by + \bA'\bx=\cmid{[\bA]}\bx.
$$
When $\bx\in[\bx]$ and $\bA\in[\bA]$, the value of $\bA'\bx$ ranges in the interval
$$
[\bz]:=[-1,1]\RAD[\bA]\cmag{[\bx]}.
$$
% $$
% [\bz]:=[-1,1]\RAD[\bA]{\jp\cdot\RAD[\bx]}.
% $$
Thus it is sufficient to find $[\by]$ such that
$$
\Omega_{\forall\exists}\left(\cmid{[\bA]}, [\by]\right)+[\bz]
 \subseteq \Omega_{\forall\exists}\left(\cmid{[\bA]}, [\bx]\right).
$$
That is, for each $\by\in[\by]$ and $\bz\in[\bz]$, there must be $\bx\in[\bx]$ such that
$$
\cmid{[\bA]}\by + \bz = \cmid{[\bA]}\bx,
$$
or, equivalently
$$
\bz = \cmid{[\bA]}(\bx-\by).
$$
If $\cmid{[\bA]}$ has full row rank, this equation can be deduced from
$$
\cmid{[\bA]}^\dag \bz = \bx-\by.
$$
By choosing an appropriate $\bx$, the right-hand side can attain any vector in $[-1,1](1-r)\RAD[\bx]$. Since the interval vector $[\bz]$ is symmetric around zero, it is sufficient to compare the radii
$$
|\cmid{[\bA]}^\dag| \RAD[\bz] \leq (1-r)\RAD[\bx],
$$
from which the statement follows.
\end{proof}

This approach to computing an inner polytope approximation of $\tolsetAx$ by a set of the form $\Omega_{\forall\exists}\left(\cmid{[\bA]}, [\by]\right)$ is particularly convenient when the intervals in $[\bA]$ are narrow relatively to those in $[\bx]$. In that case, the value of $r$ is close to~1, and the inner approximation is tight.
% Furthermore, a translation plus scaling procedure can be utilized to improve the estimate of $r$.
% The linear system $\tolsetAx$ is translated to force $\cmid{[\bx]} = {\bf 0}$ and scaled by a large scaling value $f$ by putting $\Omega_{\forall\exists}\left([\bA], f \cdot ([\bx] - \cmid{[\bx]})\right)$.
% {\mh[MH: Are you sure? In (\ref{eq:tolsetAx_r_polytope}), the value $f$ is in both numerator and denominator, so it cancels. I am also not sure about the translation -- notice that interval arithmetic is not distributive, so how you perform it?]}
% {\jp Indeed $f$ should have no effect. In fact what I proposed is equivalent to
% \begin{equation*}
% \begin{aligned}
% &r=\min_{i=1,\dots,n}\left(
%   \frac{\left(\RAD[\bx]-|\cmid{[\bA]}^\dag|\cdot\RAD[\bA]
%     \cdot\RAD{[\bx]}\right)_i}
%   {\RAD[\bx]_i} \right),
% \end{aligned}
% \end{equation*}
% from which $\RAD[\bx]$ cancel and $r$ becomes a function of only $[\bA]$. Intuitively, this seems to make sense as $r$ is proportional to the uncertainties in $[\bA]$ and is not affected by the size or position of the box $[\bx]$. I need to verify this though.
% }
% This increases the width of the intervals in $[\bx]$, thereby improving the evaluation of $r$.
Inner approximations for the largest inscribed n-ball and n-cube in $\tolsetAx$ may then be computed from the inner polytope approximation $\Omega_{\forall\exists}\left(\cmid{[\bA]}, [\by]\right)$
using Theorems~\ref{th:r_cube_tolsetx} and~\ref{th:r_ball_tolsetx} respectively.

As an example, given $\bA$ and $[\bx]$ from~\eqref{eq:example_values} and adding uncertainties of $[-0.01,0.01]$ to each matrix element gives the following inner approximation of $\tolsetAx$ shown in Figure~\ref{fig:error_int_plot}:
\begin{itemize}
    \item the associated polytope with $r=0.9289$ (purple);
    \item the largest n-cube centered at the origin with $r=0.4353$ (red);
    \item the largest n-ball centered at the origin with $r=0.5951$ (green);
    \item the largest n-cube with variable center with $b_c = (0.4510,0.2622)$ and $r=0.4353$ (dashed line);
    \item the largest n-ball with variable center with $b_c = (0.2714,0.1577)$ and $r=0.5951$ (dashed line).
\end{itemize}
For comparison the polytope $\tolsetmidAx$ is also shown which necessarily contains the polytope $\tolsetAx$ due to~\eqref{eq:linear_system_b_inter}.

\begin{figure}[h!]
    \centering
    \includegraphics[width=0.4\textwidth,trim={0cm 0.0cm 1cm 0.0cm},clip]{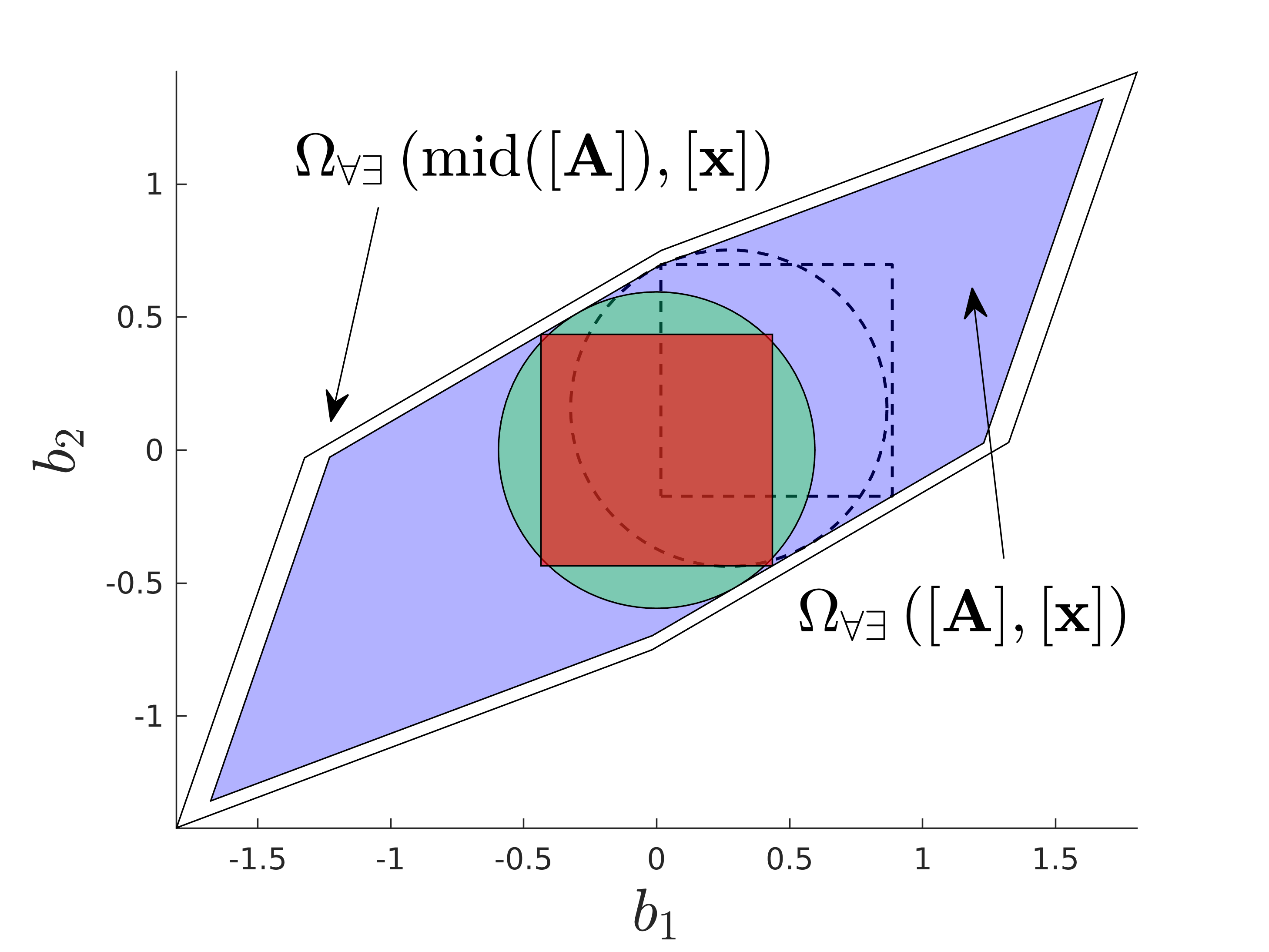}
    \caption{The associated polytope, n-cube, and n-ball inner approximations for $\tolsetAx$. For comparison the polytope $\tolsetmidAx$ is also shown.}
    \label{fig:error_int_plot}
\end{figure}

\subsection{Inner Approximations of $\tolsetAb$}\label{sec:tolsetAb}
According to~\cite{Rohn2006}, a vector $\bx$ belongs to $\tolsetAb$ (\emph{i.e.}, is a tolerance solution)
if an only if
\begin{equation}\label{eq:rohn_tolerable_set}
\begin{aligned}
\left| \cmid{[\bA]} \bx -\cmid{[\bb]}\right| \le -\crad{[\bA]}|\bx| + \crad{[\bb]}
\end{aligned}
\end{equation}
This allows to quickly verify if a desired property exists locally.
Furthermore, the following theorems are proposed for directly computing inner approximations of the local capabilities, which may be used for real-time applications. The problems of computing outer approximations of the tolerable solution set have been addressed in~\cite{PopHla2013}.

\begin{theorem}[Largest inscribed n-cube centered at ${\bx_c}$]
\label{th:n-cube}
If ${\bx_c}$ belongs to $\tolsetAb$ then an n-cube
\begin{equation}\label{eq:n_cube}
\begin{aligned}
% \mathcal{C} = {\bx_c} + r[\be]
~[\bc] = {\bx_c} + r[\be]
\end{aligned}
\end{equation}
with
\begin{equation} \label{eq:tolsetAb_cube_r}
\begin{aligned}
r=\min_{i =1,\ldots,m} \cinf{\frac{\crad{b_i} - \left| \cmid{[b_i]} - [\ba_i] {\bx_c}\right|}{\|[\ba_i]\|_1}}
\end{aligned}
\end{equation}
where $[\ba_i]$ is the $i^{th}$ row of $[\bA]$, is contained in $\tolsetAb$.
\end{theorem}
The proof for Theorem~\ref{th:n-cube} is given in~\cite{SHARY199553}. Notice that it provides the largest n-cube centered in ${\bx_c}$. Notice also that it requires that the right-hand side of (\ref{eq:tolsetAb_cube_r}) is evaluated exactly.
 If we evaluate it by interval arithmetic, it provides a lower bound on the largest cube only.
  %\vp{Not clear what you mean by this. May be add a ref?},
  %{\mh[ I mean that it depends how to formula is evaluated. If by interval arithmetic, then it is overestimated and we have a lower bound. If exactly (which is a hard optimization problem in variable $\ba_i$), then we the largest cube.)]}
In~\cite{Pop2013a}, the problems of computing inner approximations of the tolerable solution set are also considered.

On the other hand, the test $[\bc]\subseteq\tolsetAb$ can be simply performed by interval evaluation $[\bA][\bc]\subseteq[\bb]$. Therefore, we can extend the above inner n-cube $[\bc]$ by $\varepsilon$-inflation approach to $[\bc]:=[\bc]+\varepsilon[\be]$ and repeat it while the test does not break.
% [I would rather use e.g. $[\by]$ instead of  $\mathcal{C}$ because it is indeed an interval vector.]

To overcome the above drawbacks, we propose a more general method. It finds the optimal center point of the n-cube and it computes exactly its radius, too.
\begin{theorem}[Largest inscribed n-cube with variable center]
\label{th:n-cube_variable}
The largest inscribed n-cube to $\tolsetAb$ reads $[\bc] = \bx_c + r[\be]$, where $\bx_c ,r$ is an optimal solution to
\begin{subequations}\label{eq:lpMaxInnrCubeVar}
\begin{align}
\max \ \              &r\\
\text{subject to}\ \  &\sum_j y^i_j\leq\onum{b}_i,\quad\forall i,\\
 & y^i_j\geq \onum{a}_{ij}({x_c}_j+r),\\
 & y^i_j\geq \onum{a}_{ij}({x_c}_j-r),\\
 & y^i_j\geq \unum{a}_{ij}({x_c}_j+r),\\
 & y^i_j\geq \unum{a}_{ij}({x_c}_j-r),\\
 &\sum_j z^i_j\geq\unum{b}_i,\quad\forall i,\\
 & z^i_j\leq \onum{a}_{ij}({x_c}_j+r),\\
 & z^i_j\leq \onum{a}_{ij}({x_c}_j-r),\\
 & z^i_j\leq \unum{a}_{ij}({x_c}_j+r),\\
 & z^i_j\leq \unum{a}_{ij}({x_c}_j-r).
\end{align}
\end{subequations}
\end{theorem}

\begin{proof}
Our problem reads
\begin{align*}
\max\ \  &r\\
\text{subject to}\ \ &  \csup{[\ba_i](\bx_c + r[\be])} \leq\onum{b}_i \quad\forall i,\\
&  \cinf{[\ba_i](\bx_c + r[\be])} \geq\unum{b}_i\quad\forall i.
\end{align*}
Denoting by $y^i_j$ and $z^i_j$ the supremum and the infimum of the product $[a_{ij}]({x_c}_j + r[-1,1])$, respectively, the resulting optimization model follows.
\end{proof}

Notice that \eqref{eq:lpMaxInnrCubeVar} is a linear programming problem, so it can be solved efficiently. The number of variables is $1+n+2mn$, so the problem is still polynomially solvable.

\begin{theorem}[Largest inscribed n-ball centered at ${\bx_c}$]
\label{th:n-ball}
If ${\bx_c}$ belongs to $\tolsetAb$, then the n-ball
\begin{equation}\label{eq:n_ball}
\begin{aligned}
\mathcal{B} = \{ \bx ~ | ~ \|\bx - {\bx_c}\| \le r\}
\end{aligned}
\end{equation}
with
\begin{equation} \label{eq:tolsetAb_ball_r}
\begin{aligned}
r&=\min_{i =1,\ldots,m}\\
&\MIN\left(\cinf{\frac{-[\ba_i] {\bx_c} + \onum{b}_i}{\|[\ba_i]\|_2}},\,\cinf{\frac{[\ba_i] {\bx_c} - \unum{b}_i}{\|[\ba_i]\|_2}}\right)
\end{aligned}
\end{equation}
 where ${\|\cdot\|_2}$ is the 2-norm, is contained in $\tolsetAb$.
\end{theorem}

\begin{proof}
According to \cite{Rohn2006}, we know that any $\bx \in \tolsetAb$ must satisfy
\begin{equation}
\begin{aligned}
\{\bA \bx ~|~ \bA\in[\bA]\} \subseteq [\underline{\bb},\overline{\bb}]
\end{aligned}
\end{equation}
Therefore, $\tolsetAb$ is described by the convex polytope given by the following inequalities
% {\mh[From now on it is rather informal since you operate with interval $[\ba_i]$ in the optimization problem.]}
\begin{equation}\label{eq:tolerance_solution_set}
    \begin{aligned}
        &\csup{[\ba_i]\bx} \le \onum{b}_i,\\
        &\csup{-[\ba_i]\bx} \le -\unum{b}_i,\\
        &\bx \in \mathbb{R}^n\\
        & i=1,\ldots,m\\
    \end{aligned}
\end{equation}
% where $[\ba_i]$ is the $i^{th}$ row of $[\bA]$ and $b_i$ is the $i^{th}$ element of $\bb$.
where $\SUP$ ensures the inequalities are well-defined.
In other words,~\eqref{eq:tolerance_solution_set} describes the intersection of a set of half-spaces.
% where $[\ba_i]$ gives the half-space normal vector and $b_i$ gives the half-space offset in the direction of the normal vector.

For a given $\bA\in[\bA]$, the largest n-ball centered at ${\bx_c}$ and inscribed to $\tolsetb$ is computed by the linear program
\begin{equation*}
    \begin{aligned}
        &\max              & &r\\
        &\text{subject to} & &{\ba_i {\bx_c} + r \|\ba_i\|_2} \le \onum{b}_i,\quad\forall i,\\
        & & &{-\ba_i {\bx_c} + r \|\ba_i\|_2} \le -\unum{b}_i,\quad\forall i,\\
        & & & r \ge 0,\\
    \end{aligned}
\end{equation*}
from which
\begin{equation*}
r=\min_{i =1,\ldots,m}
\MIN\left({\frac{-\ba_i {\bx_c} + \onum{b}_i}{\|\ba_i\|_2}},\,{\frac{\ba_i {\bx_c} - \unum{b}_i}{\|\ba_i\|_2}}\right).
\end{equation*}
Now, for the radius of the largest n-ball centered at ${\bx_c}$ inscribed to $\tolsetAb$ we just minimize subject to $\ba_i\in[\ba_i]$.

% \begin{comment}
% The Chebyshev center ${\bx_c}$ of a polytope can be determined by the linear programming problem which reads
% \begin{equation}\label{eq:tolsetAb_chebyshev_center}
%     \begin{aligned}
%         &\max              & &r\\
%         &\text{subject to} & &\csup{[\ba_i] {\bx_c} + r \|[\ba_i]\|_2} \le \onum{b}_i,\quad\forall i,\\
%         & & &\csup{-[\ba_i] {\bx_c} + r \|[\ba_i]\|_2} \le -\unum{b}_i,\quad\forall i,\\
%         & & & r \ge 0\\
%     \end{aligned}
% \end{equation}
% Again, $\SUP$ ensures the inequalities are well-defined.
% The maximum value for $r$ may be computed directly for a given ${\bx_c}$ as~\eqref{eq:tolsetAb_ball_r}.
% The optimal solution $\bx_c ,r$ for the largest inscribed n-ball to $\tolsetAb$ is obtained with
% \todo[inline]{There is something incorrect with this LP when considering [A]. The exact solution can be found but not with all constraints 39c, 39d, 39f, 39g simultaneously enforced.}
% \begin{subequations}\label{eq:lpMaxInnrBallVar}
% \begin{align}
% \max \ \              &r\\
% \text{subject to}\ \  &\sum_j (u^i_j) +r~ \csup{\|[\ba_i]\|_2}\leq\onum{b}_i,\quad\forall i,\\
%  & u^i_j\geq \onum{a}_{ij}{x_c}_j,\\
%  & u^i_j\geq \unum{a}_{ij}{x_c}_j,\\
%  &\sum_j (v^i_j) +r ~ \csup{\|[\ba_i]\|_2}\leq -\unum{b}_i,\quad\forall i,\\
%  & v^i_j\geq -\onum{a}_{ij}{x_c}_j,\\
%  & v^i_j\geq -\unum{a}_{ij}{x_c}_j.
% \end{align}
% \end{subequations}
% where $u^i_j$ and $v^i_j$ are the supremums of $[\ba_i] {\bx_c}$ and $-[\ba_i] {\bx_c}$, respectively.
% \end{comment}
\end{proof}

The theorem provides the largest n-ball which is centered in a given point~${\bx_c}$. Notice also that it requires that the right-hand side of (\ref{eq:tolsetAb_ball_r}) is evaluated exactly.
If we evaluate it by interval arithmetic, it provides a lower bound on the largest ball only.
% {\mh [Open problem: can be computed exactly?]}

\begin{lemma}
The function $f(a)=\frac{a^Tx+b}{\|a\|}$ is quasiconcave on the domain $a\not=0$, $a^Tx+b\geq0$.
\end{lemma}

\begin{proof}
Let $\alpha\geq0$ be fixed. The set of $a\not=0$ satisfying $f(a)\geq \alpha$ is characterized by
$$
a^Tx+b \geq \alpha \|a\|,
$$
which describes a convex set.
\end{proof}

The elements of the right-hand side of (\ref{eq:tolsetAb_ball_r}) fulfill the assumptions of the lemma. Since quasiconcave functions attain the minima on the border of the domain, we have the corollary below; the second item follows from the fact that (\ref{eq:tolsetAb_ball_r}) is monotone in some cases.

\begin{corollary}\label{corTolsetAbVert}
For every $i,j$ we have:
\begin{enumerate}[(1)]
\item
The right-hand side of (\ref{eq:tolsetAb_ball_r}) is attained for $a_{ij}\in\{\unum{a}_{ij},\onum{a}_{ij}\}$.
\item
If ${x_c}_j\geq0$ and ${a}_{ij}\geq0$, then we can fix $a_{ij}:=\onum{a}_{ij}$ in the first infimum of~(\ref{eq:tolsetAb_ball_r}).
\item
If ${x_c}_j\leq0$ and ${a}_{ij}\leq0$, then we can fix $a_{ij}:=\unum{a}_{ij}$ in the first infimum of~(\ref{eq:tolsetAb_ball_r}).
\item
If ${x_c}_j\leq0$ and ${a}_{ij}\geq0$, then we can fix $a_{ij}:=\onum{a}_{ij}$ in the second infimum of~(\ref{eq:tolsetAb_ball_r}).
\item
If ${x_c}_j\geq0$ and ${a}_{ij}\leq0$, then we can fix $a_{ij}:=\unum{a}_{ij}$ in the second infimum of~(\ref{eq:tolsetAb_ball_r}).

\end{enumerate}
\end{corollary}

Therefore, we may quickly check if ${\bx_c}$ belongs to $\tolsetAb$ by using~\eqref{eq:rohn_tolerable_set}. If so, the
corresponding inscribed n-cube and n-ball are computed directly from~\eqref{eq:tolsetAb_cube_r} and~\eqref{eq:tolsetAb_ball_r} respectively.
The proposed theorems are also valid for the $\tolsetb$ sub-problems.

\begin{theorem}[Inscribed n-ball with variable center]
Denote $s_i\coloneqq \csup{ \|[\ba_i]\|_2} $ $\forall i$.
Let $r,{\bx_c},\by$ be an optimal solution of the linear programming problem
\begin{equation} \label{eq:tolsetAb_ball_r2}
    \begin{aligned}
        &\max              & &r\\
        &\text{subject to} & &\cmid{\ba_i} \bx_c+\crad{\ba_i} \by + rs_i\le \onum{b}_i,\quad\forall i,\\
        & & &-\cmid{\ba_i} \bx_c+\crad{\ba_i} \by + r s_i \le -\unum{b}_i,\quad\forall i,\\
        & & & r \ge 0,\  \by\ge \bx_c,\ \by\ge -\bx_c.
    \end{aligned}
\end{equation}
Then the ball \eqref{eq:n_ball} is contained in $\tolsetAb$.
\end{theorem}

\begin{proof}
As in the proof of Theorem~\ref{th:n-ball}, $\tolsetAb$ is described by \eqref{eq:tolerance_solution_set}.
% where $[\ba_i]$ gives the half-space normal vector and $b_i$ gives the half-space offset in the direction of the normal vector.
The Chebyshev center ${\bx_c}$ of a polytope can be determined by the linear programming problem which reads
\begin{equation}\label{eq:tolsetAb_chebyshev_center}
    \begin{aligned}
        &\max              & &r\\
        &\text{subject to} & &\csup{[\ba_i] {\bx_c} + r \|[\ba_i]\|_2} \le \onum{b}_i,\quad\forall i,\\
        & & &\csup{-[\ba_i] {\bx_c} + r \|[\ba_i]\|_2} \le -\unum{b}_i,\quad\forall i,\\
        & & & r \ge 0\\
    \end{aligned}
\end{equation}
It is hard to express the suprema explicitly, so we estimate them from above, resulting in a lower bound to the maximum inscribed n-ball
\begin{equation}
    \begin{aligned}
        &\max              & &r\\
        &\text{subject to} & &\csup{[\ba_i] {\bx_c}} + rs_i \le \onum{b}_i,\quad\forall i,\\
        & & &\csup{-[\ba_i] {\bx_c}} + r s_i \le -\unum{b}_i,\quad\forall i,\\
        & & & r \ge 0\\
    \end{aligned}
\end{equation}
Since $\csup{[\ba_i] {\bx_c}}=\cmid{\ba_i} \bx_c+\crad{\ba_i}|\bx_c|$, we equivalently have \eqref{eq:tolsetAb_ball_r2}, where $\by$ substitutes for $|\bx_c|$.
\end{proof}

By the first item of Corollary~\ref{corTolsetAbVert}, the largest inscribed n-ball is attained for a vertex of $\bA$. This is true not only for the fixed  center $\bx_c$, but even when the center is variable (because we can fix the center at the optimal one). This yields the following method of exponential complexity (with respect to the size of $[\bA]$). It remains an open question if the problem is computationally tractable or not.

\begin{theorem}[Largest inscribed n-ball with variable center]
Let $r,{\bx_c}$ be an optimal solution of the linear programming problem
\begin{equation}\label{eq:lpMaxInnrBallVar}
    \begin{aligned}
        &\max              & &r\\
        &\text{subject to} & &\ba_i \bx_c+ r\|\ba_i\|_2\le \onum{b}_i,\quad\forall i,\forall \ba_i \in \vertex([\ba_i]),\\
        & & &-\ba_i \bx_c+ r\|\ba_i\|_2\le -\unum{b}_i,\quad\forall i,\forall \ba_i \in \vertex([\ba_i]),\\
        & & & r \ge 0.
    \end{aligned}
\end{equation}
Then the ball \eqref{eq:n_ball} is the largest ball contained in $\tolsetAb$.
\end{theorem}

The polytope $\tolsetAb$ may be represented as the set of all feasible solutions of a linear programming problem. The corresponding linear program~\eqref{eq:rohninnersolution} was first presented in~\cite{Rohn1986157158}.
% Formula~\eqref{eq:rohn_tolerable_set} may also be used to compute the exact solution as a convex polytope within each orthant, allowing to remove the absolute value $|\bx|$.
\begin{theorem}\label{th:rohninnersolution}
$\bx \in \tolsetAb$ if and only if $\bx = \bx_1 - \bx_2$ is a solution to the system of linear inequalities
\begin{equation}\label{eq:rohninnersolution}
\begin{aligned}
 \onum{\bA} \bx_1 - \unum{\bA} \bx_2 &\le \onum{\bb}\\
 -\unum{\bA} \bx_1 + \onum{\bA} \bx_2 &\le \unum{\bb}\\
 \bx_1 \ge 0,~\bx_2 &\ge 0
\end{aligned}
\end{equation}
\end{theorem}
The proof for Theorem~\ref{th:rohninnersolution} is given in~\cite{Rohn1986157158}.

For example, given $\bA$ as the transpose from~\eqref{eq:example_values} and adding uncertainties of $[-0.01,0.01]$ to each matrix element and an arbitrary non-symmetric box $[\bb]$
\begin{equation}\label{eq:arbitrary_b}
    [\bb] =
    \begin{pmatrix}
[  -74.0,   95.0] \\
[  -24.0,   20.0] \\
[  -22.0,   33.0]
\end{pmatrix}
\end{equation}
the following inner approximations of $\tolsetAb$ are given in Figure~\ref{fig:tolsetAb}:
\begin{itemize}
    \item the associated polytope (purple);
    \item the largest n-cube centered at the origin with $r=18.5082$ (red);
    \item the largest n-ball centered at the origin with $r=25.3333$ (green);
    \item the largest n-cube with variable center with $x_c = (-1.0599,-3.1972)$ and $r=20.3591$ (dashed line);
    \item the largest n-ball with variable center with $x_c = (-2.5291, -0.6963)$ and $r=27.8666$ (dashed line).
\end{itemize}
For comparison the polytope $\tolsetmidAb$ is also shown which necessarily contains the polytope $\tolsetAb$ due to~\eqref{eq:linear_system_x_inter}.

\begin{figure}[h!]
    \centering
    \includegraphics[width=0.4\textwidth,trim={0cm 0.0cm 1cm 0.0cm},clip]{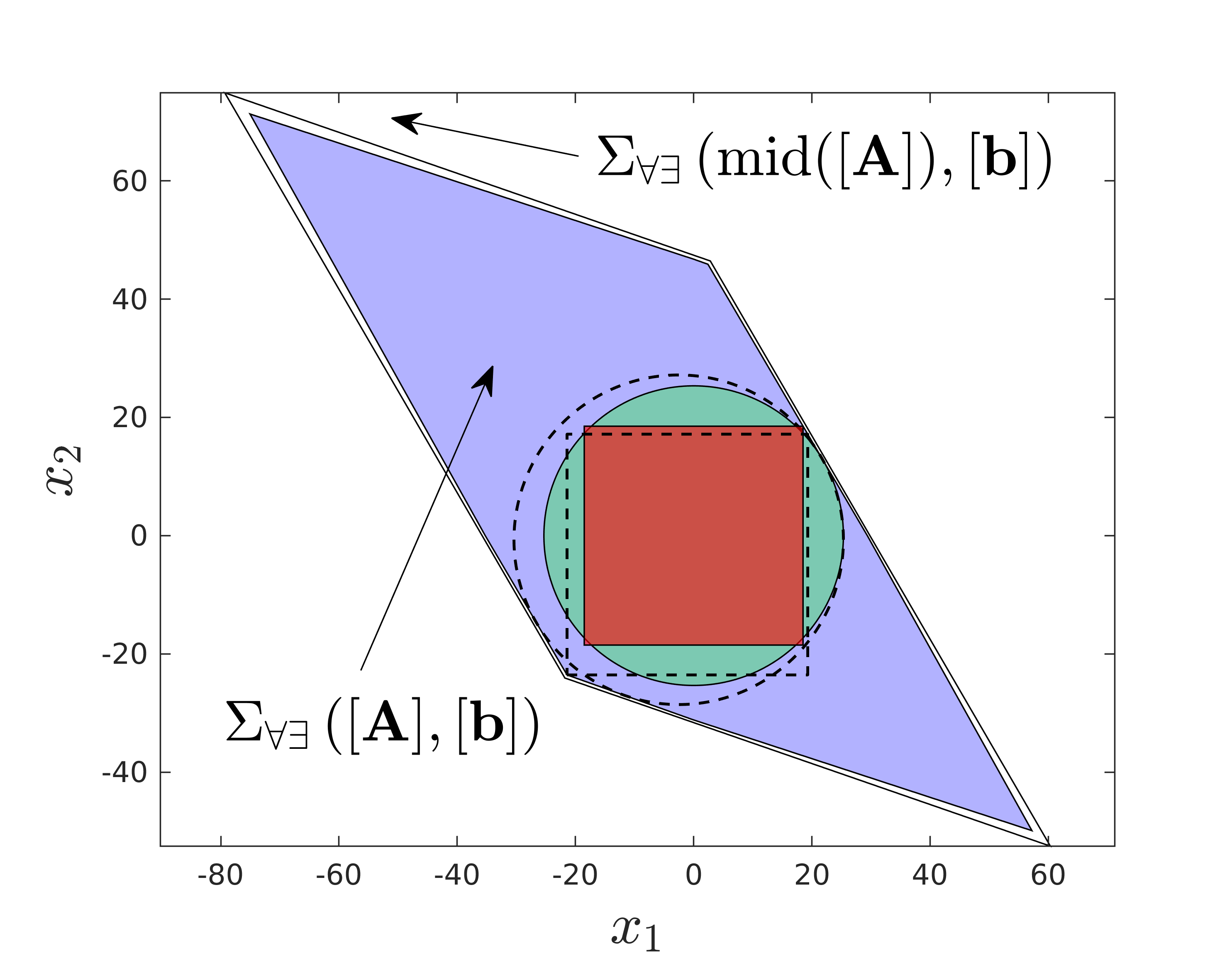}
    \caption{The associated polytope, n-cube, and n-ball inner approximations of $\tolsetAb$. For comparison the polytope $\tolsetmidAb$ is also shown.}
    \label{fig:tolsetAb}
\end{figure}

\subsection{$\tolsetb$ Inner Approximations}\label{sec:tolsetb}
When $\bA$ is square, the $\tolsetb$ polytope can be obtained as the convex hull of the set of $\bx$ associated with the $2^m$ vertices of $[\bb]$ as
\begin{equation}\label{eq:tolsetb_square}
\tolsetb = \conv(\{\bx~|~ (\exists \bb \in \vertex([\bb]))(\bx = \bA^{-1}\bb)\})
\end{equation}

When $\bA$ is non-square and $m>n$ the linear system of equations is over-determined and the Moore--Penrose pseudoinverse solution
\begin{equation}
\begin{aligned}
\bx = \bA^{\dag}\bb,
\end{aligned}
\label{eq:pseudo-inv}
\end{equation}
where $\bA^{\dag} = (\bA^T \bA)^{-1} \bA^T$,
only ensures that the norm of the error $\bb -  \bA \bx$ is minimized.
This error is zero and \eqref{eq:pseudo-inv} has a solution if and only if the set of $\bb$ belongs to the image ($\image$) of $\bA$.
In \cite{Chiacchio1996} the authors propose finding a reduced polytope that is given by the intersection $[\bb] \cap \image(\bA)$ and apply this method to solve the wrench problem of redundant serial manipulators. This approach has been recently extended to allow for online computation of the polytope \cite{skuric2020online}.
For many robotics applications, this is equivalent to considering only the set of $\bb$ that do not cause internal motions.
For example, when considering the velocity kinematics model, only the joint states resulting in motions at the end-effector are considered, as states in the null space produce only internal motions.
% \dd{for this particulat example, can we give a reference for a particular problem of the Table I, II}

The associated $\bx$ for each vertex of the reduced polytope can be computed from \eqref{eq:pseudo-inv} and the convex hull of the resulting set gives the polytope $\tolsetb$.
The $\tolsetb$ polytope for over-determined systems is given by
\begin{equation}\label{eq:tolsetb_nonsquare}
\begin{aligned}
&\tolsetb = \\
&\conv(\{\bx~|~ (\exists \bb \in \vertex([\bb] \cap \image(\bA)))(\bx = \bA^{\dag}\bb)\})
\end{aligned}
\end{equation}
The proposed n-cube and n-ball theorems (Theorems~\ref{th:n-cube}, \ref{th:n-cube_variable}, and \ref{th:n-ball}) are also valid and applicable to the $\tolsetb$ sub-problems.

For example, given $\bA$ as the transpose from~\eqref{eq:example_values} and $[\bb]$ from~\eqref{eq:arbitrary_b},
the box $[\bb]$ and corresponding reduced polytope $[\bb] \cap \image(\bA)$, and
the following inner approximations of $\tolsetb$ are given in Figure~\ref{fig:tolsetb}:
\begin{itemize}
    \item the associated polytope (purple);
    \item the largest n-cube centered at the origin with $r=18.8573$ (red);
    \item the largest n-ball centered at the origin with $r=25.7800$ (green);
    \item the largest n-cube with variable center with $x_c = (-1.0450,-3.3319)$ and $r=20.7430$ (dashed line);
    \item the largest n-ball with variable center with $x_c = (-2.5624, -0.7216)$ and $r=28.3579$ (dashed line).
\end{itemize}

% A =
%     0.8947    0.3348
%     0.6707    0.3899
%     0.2409    0.6958
% intval b =
% [  -74.2503,   93.2065]
% [  -24.4139,   20.8148]
% [  -22.9663,   33.9577]
% xc =
%      0
%      0
% r =
%   25.78
\begin{figure}[h!]
    \centering
    \includegraphics[width=0.4\textwidth,trim={0cm 0.0cm 1cm 0.0cm},clip]{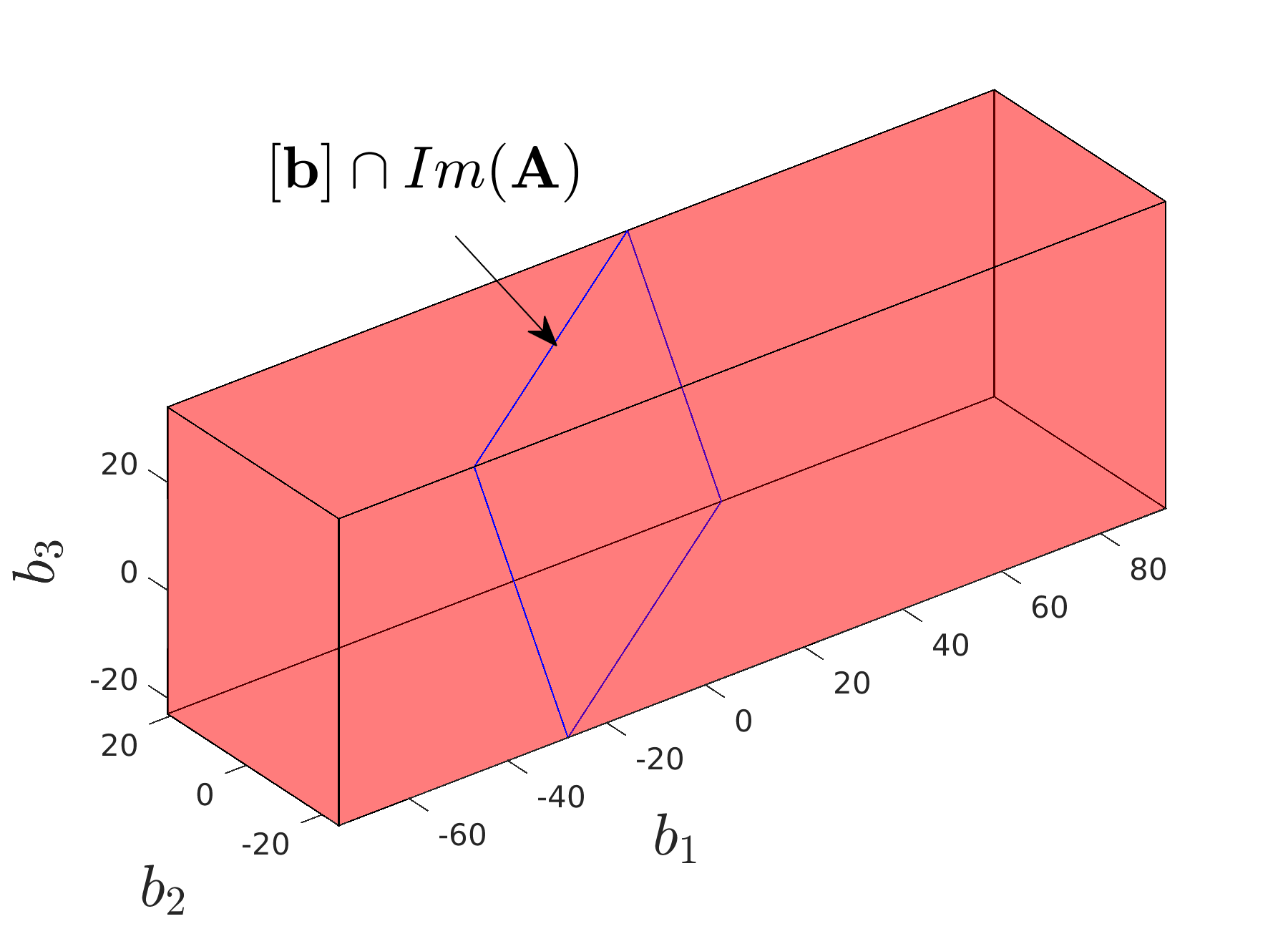}
    \includegraphics[width=0.4\textwidth,trim={0cm 0.0cm 1cm 0.0cm},clip]{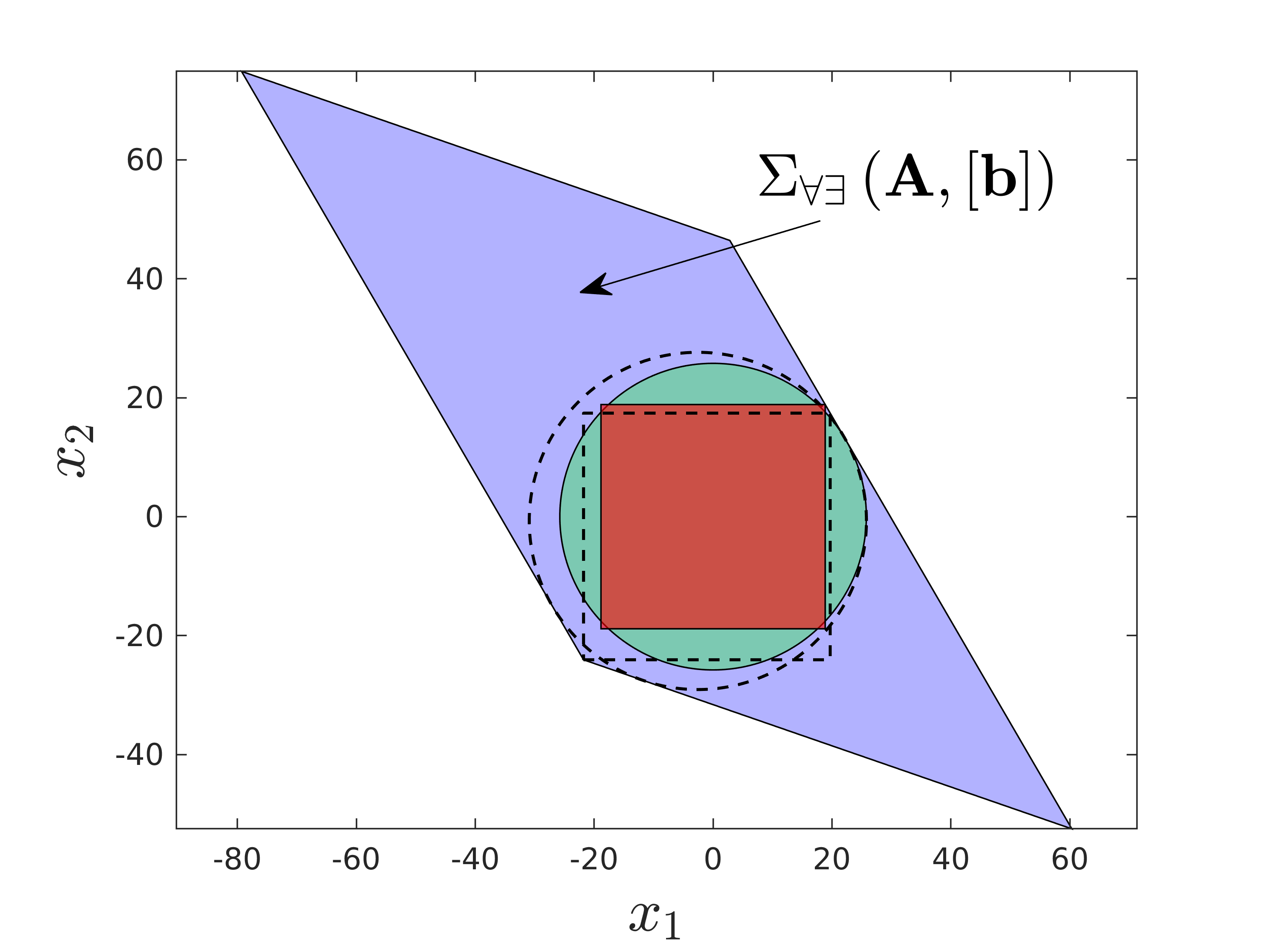}
    \caption{(Top) The box $[\bb]$ and the corresponding reduced polytope $[\bb] \cap \image(\bA)$. (Bottom) The associated polytope, n-cube, and n-ball inner approximations of $\tolsetb$.}
    \label{fig:tolsetb}
\end{figure}

\subsection{Inner Approximations Summary}
The equations associated with each inner approximation for each class of problem are summarized in Table~\ref{tab:summary}.
Software implementations of the proposed inner approximation algorithms are made available at~\cite{Pickard2020} for Matlab.
To get an estimate of the computational times for typical 7-degree-of-freedom redundant robotics applications, for each problem the mean computational times\footnote{A desktop computer with a AMD Phenom(tm) II X6 1045T 2.70 GHz processor and 16.0 GB of RAM is used to compute mean computational times.} and standard deviations (over 20 runs) evaluated for a random system with $n=7$ and $m=6$ are provided in Table~\ref{tab:summary}.
The fixed center point is taken as the origin.
It can be noted that the Matlab implementations are not optimized for speed and the computational times can be significantly reduced with more efficient implementations. However, the already low computational times for many of the inner approximations demonstrates the suitability for real-time applications.

\begin{table*}[t]
\setlength{\fboxsep}{0pt}%
\setlength{\fboxrule}{0pt}%
\begin{center}
\caption{Inner approximation summary.\label{tab:summary}}
\begin{tabular}{|l|l|l|l|l|l|l|l|l|}
     \hline
     {\bf Approximations} &$\tolsetx$& ms $\pm$ std & $\tolsetAx$ & ms $\pm$ std& $\tolsetb$& ms $\pm$ std & $\tolsetAb$ & ms$\pm$ std\\\hline
     Polytope inner
     & \eqref{eq:tolsetx}
     & 56.842
     & \eqref{eq:tolsetAx_r_polytope} + \eqref{eq:tolsetx}
     & 56.719
     & \eqref{eq:tolsetb_square} or \eqref{eq:tolsetb_nonsquare}
     & 187.374
     & \eqref{eq:rohninnersolution}
     & 381.568\\
     approximation
     && $\pm$1.090
     && $\pm$0.430
     && $\pm$0.982
     && $\pm$3.544
     \\\hline
     Largest inscribed n-cube
     & \eqref{eq:tolsetx} + \eqref{eq:tolsetx_cube_r}
     & 1.496
     & \eqref{eq:tolsetAx_r_polytope} + \eqref{eq:tolsetx_cube_r}
     & 1.537
     & \eqref{eq:tolsetAb_cube_r}
     & 0.024
     & \eqref{eq:tolsetAb_cube_r}
     & 0.024\\
     centered at a point
     && $\pm$0.033
     && $\pm$0.030
     && $\pm$0.002
     && $\pm$0.001
     \\\hline
     Largest inscribed n-cube
     & \eqref{eq:tolsetx} + \eqref{eq:tolsetx_chebyshev_center_cube}
     & 20.748
     & \eqref{eq:tolsetAx_r_polytope} + \eqref{eq:tolsetx_chebyshev_center_cube}
     & 20.807
     & \eqref{eq:lpMaxInnrCubeVar} -- exact
     & 30.690
     & \eqref{eq:lpMaxInnrCubeVar} -- exact
     & 33.154\\
     with variable center
     && $\pm$0.389
     && $\pm$0.310
     && $\pm$0.414
     && $\pm$0.675
     \\\hline
     Largest inscribed n-ball
     & \eqref{eq:tolsetx}  + \eqref{eq:tolsetx_ball_r}
     & 1.509
     & \eqref{eq:tolsetAx_r_polytope} + \eqref{eq:tolsetx_ball_r}
     & 1.541
     & \eqref{eq:tolsetAb_ball_r}
     & 0.036
     & \eqref{eq:tolsetAb_ball_r}
     & 0.036\\
     centered at a point
     && $\pm$0.029
     && $\pm$0.023
     && $\pm$0.004
     && $\pm$0.002
     \\\hline
     Largest inscribed n-ball
     & \eqref{eq:tolsetx} + \eqref{eq:tolsetx_chebyshev_center_ball}
     & 19.872
     & \eqref{eq:tolsetAx_r_polytope} + \eqref{eq:tolsetx_chebyshev_center_ball}
     & 19.883
     & \eqref{eq:tolsetAb_ball_r2}
     & 17.287
     & \eqref{eq:tolsetAb_ball_r2}
     & 17.560\\
     with variable center
     && $\pm$0.322
     && $\pm$0.294
     && $\pm$0.181
     && $\pm$0.434\\
     &&
     &&&\eqref{eq:lpMaxInnrBallVar} -- exact
     & 34.701
     & \eqref{eq:lpMaxInnrBallVar} -- exact
     & 48.459\\
     &&
     &&
     && $\pm$ 0.466
     && $\pm$0.481\\
     \hline
\end{tabular}
\end{center}
\end{table*}

\section{Applications}~\label{sec:applications}
One application of this work is to improve the reliability of the computed capabilities, since all sources of error can be easily modeled and managed. Methods for reliably modeling sources of error are described within the appropriate design framework~\cite{Merlet2008,PICKARD2019237}. This allows to compute certifiable results providing confidence of the calculated capabilities of a manipulator. Other applications of this work are to provide reliable inner approximations of the common capabilities of a robotic manipulator over a given time horizon. This allows to reliably compute local capabilities, in real-time in many cases, that can be used online to evaluate the manipulator's current capabilities as well as its future capabilities within a time horizon.

As a simple example, consider a redundant three-link planar manipulator for a positioning task with point masses $m_1,m_2,m_3$ at distal ends of links of lengths $l_1,l_2,l_3$. The manipulator is depicted in Figure~\ref{fig:threelinkmanip}.
 \begin{figure}[h!]
    \centering
    \includegraphics[width=0.3\textwidth]{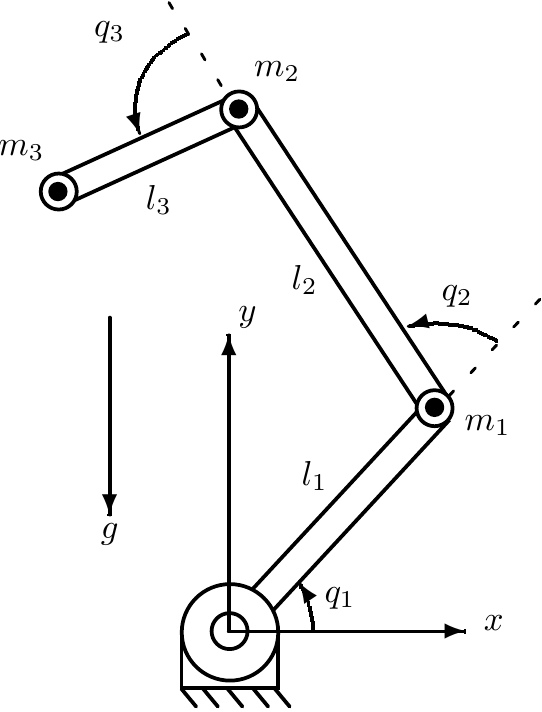}
    \caption{Three-link planar manipulator with point masses at distal ends of links.}
    \label{fig:threelinkmanip}
\end{figure}

The manipulator's $2 \times 3$ Jacobian matrix is given by
\begin{equation}
\begin{aligned}
{\bJ}(\bq) =&
    \begin{pmatrix}
    J_1-sq_{12} l_2-sq_1 l_1 &
    J_1-sq_{12} l_2 &
    J_1\\
    J_2+cq_{12} l_2+cq_1 l_1 &
    J_2+cq_{12} l_2 &
    J_2\\
    \end{pmatrix}
\end{aligned}
\end{equation}
where $q_{ijk} = (q_i+q_j+q_k)$, $sq_{i} = \sin(q_{i})$, $cq_{i} = \cos(q_{i})$,
$J_1 = -sq_{123} l_3$, and
$J_2 = cq_{123} l_3$.

The $2 \times 3$ time derivative of the Jacobian matrix is given by
\begin{equation}
\begin{aligned}
\dot{\bJ}(\bq,\dot{\bq}) =&
    \begin{pmatrix}
    -J_6 -J_3- cq_1 l_1 \dot{q}_{1} &
    -J_6-J_3&
    -J_6\\
    J_5 -J_4-sq_1 l_1 \dot{q}_{1}&
    J_5 -J_4&
    J_5\\
    \end{pmatrix}
\end{aligned}
\end{equation}
where $\dot{q}_{ijk} = \dot{q}_i+ \dot{q}_j + \dot{q}_k$, $J_3 = cq_{12} l_2 \dot{q}_{12}$,
$J_4 = sq_{12} l_2 \dot{q}_{12}$,
$J_5 = J_1 \dot{q}_{123}$, and
$J_6 = J_2 \dot{q}_{123}$.

The mass matrix, centrifugal and Coriolis vector, and gravity vector are given by
\begin{equation}
\begin{aligned}
{\bf M}(\bq) =&
    \begin{pmatrix}
    K_7 &
    K_1+K_3+K_4 &
    K_2+K_6 \\
    K_1+K_3+K_4 &
    2 K_5+K_4 &
    K_5+K_6\\
    K_2+K_6 &
    K_5+K_6 &
    K_6\\
    \end{pmatrix}
    \\
{\bf c}(\bq,\dot{\bq}) =&
    \begin{pmatrix}
    K_{10}( K_8+K_{12})+K_{11}( K_8+K_{9}) \\
    \dot{q_1}^2(K_8 + K_{12})+K_{9} K_{11}\\
    -K_{9} K_{10}+\dot{q_1}^2( K_8 + K_{9}) \\
    \end{pmatrix}
    \\
{\bg}(\bq) =& g
    \begin{pmatrix}
    cq_{123} l_3 m_3+l_2 cq_{12} m_{23}+l_1 cq_{1} m_{123}\\
    cq_{123} l_3 m_3+l_2 cq_{12} m_{23}\\
    cq_{123} l_3 m_3\\
    \end{pmatrix}
\end{aligned}
\end{equation}
where
\begin{equation}
\begin{aligned}
K_1 &= l_1 l_2 m_{23} cq_2\\
K_2 &= l_3 m_3 (l_1 cq_{23}+l_2  cq_3)\\
K_3 &= l_3 m_3 (l_1 cq_{23}+2 l_2 cq_3)\\
K_4 &= l_2^2 m_{23}+K_6\\
K_5 &= l_2 l_3 m_3 cq_3\\
K_6 &= l_3^2 m_3\\
K_7 &= 2 K_1 + 2 K_2 +K_4+l_1^2 m_{123}\\
K_8 &= l_1 l_3 m_3 sq_{23}\\
K_{9} &= l_2 l_3 m_3 sq_3\\
K_{10} &= -(2 \dot{q}_{1} \dot{q}_{2}+\dot{q}_{2}^2)\\
K_{11} &= -(2 \dot{q}_{1} \dot{q}_{3}+2 \dot{q}_{2} \dot{q}_{3}+\dot{q}_{3}^2)\\
K_{12} &= l_1 l_2 m_{23} sq_2
\end{aligned}
\end{equation}
with $m_{ijk} = m_i+ m_j + m_k$,  and $g$ is the gravitational constant.

For the following examples, realistic parameters and joint state limits are selected to approximate a simplified planar model of joints 2, 4, and 6 of the Franka Emika Panda robot. The parameter values with assumed uncertainties are
\begin{equation}\label{eq:parametervalues}
    \begin{aligned}
        [l_1]&=0.328 \pm 0.0001~\text{m}\\
        [l_2]&=0.394 \pm 0.0001~\text{m}\\
        [l_3]&=0.1385 \pm 0.0001~\text{m}\\
        [m_1]&=1.9 \pm 0.001~\text{kg}\\
        [m_2]&=1.6 \pm 0.001~\text{kg}\\
        [m_3]&=1.3 \pm 0.001~\text{kg}\\
    \end{aligned}
\end{equation}
and the corresponding joint state limits (\emph{i.e.}, $[{\bq}_{lim}]$, $[\dot{\bq}_{lim}]$, $[\ddot{\bq}_{lim}]$, $[\boldsymbol{\tau}_{lim}]$, $[\dot{\boldsymbol{\tau}}_{lim}]$) are taken from the manufacturer's specifications
% \vp{why is the table commented ? Do you want to put it in appendix?}
as
\begin{equation}
    \begin{aligned}[]
        [{\bq}_{lim}]&=
        \begin{pmatrix}
            [-1.7628,1.7628]\\
            [-3.0718,-0.0698]\\
            [-0.0175,3.7525]
        \end{pmatrix}~\text{rad}\\
        [\dot{\bq}_{lim}]&=
        \begin{pmatrix}
            [-2.175,2.175]\\
            [-2.175,2.175]\\
            [-2.610,2.610]
        \end{pmatrix}~\text{rad/s}\\
        [\ddot{\bq}_{lim}]&=
        \begin{pmatrix}
            [-7.5,7.5]\\
            [-12.5,12.5]\\
            [-20,20]
        \end{pmatrix}~\text{rad/s}^2\\
        % [\dddot{\bq}_{lim}]&=
        % \begin{pmatrix}
        %     [-3750,3750]\\
        %     [-6250,6250]\\
        %     [-10000,10000]
        % \end{pmatrix}~\text{rad/s}^3\\
        [\boldsymbol{\tau}_{lim}] &=
        \begin{pmatrix}
            [-87,87]\\
            [-87,87]\\
            [-12,12]
        \end{pmatrix}~\text{Nm}\\
        [\dot{\boldsymbol{\tau}}_{lim}] &=
        \begin{pmatrix}
            [-1000,1000]\\
            [-1000,1000]\\
            [-1000,1000]
        \end{pmatrix}~\text{Nm/s}\\
    \end{aligned}
\end{equation}

To visualize the effects of considering a set of configurations, for
\begin{equation}\label{eq:config_eg}
\begin{aligned}[]
[{\bq}] &= (1.0, -1.0, 1.0) \pm 0.1~\text{rads}\\
\end{aligned}
\end{equation}
an approximation of the set of positions associated with the distal ends of the links is depicted in Figure~\ref{fig:threelinkpositions}.
% The edges of the various regions are non-linear and are therefore difficult to compute and handle directly.
The proposed inner approximations provide a convenient means of evaluating the robot's common capabilities over the corresponding set of poses, provided that the intersection of capabilities is not empty.

\subsection{Evaluating the Velocity Problem}
A necessary procedure when analyzing the performance of a robotic manipulator is to evaluate the possible end-effector velocities through the mapping of known joint velocity limits with \eqref{eq:velocity_problem}.
The velocity problem is of the class $\tolsetx$ and the associated inner approximations of section~\ref{sec:tolsetx} may be used to evaluate the possible end-effector velocities of the manipulator at a given configuration ${\bq}$.
When considering a set of configurations $[{\bq}]$, the manipulator's Jacobian matrix becomes an interval matrix and the associated inner approximations of section~\ref{sec:tolsetAx} may be used.

At the following configurations
\begin{equation}\label{eq:config}
\begin{aligned}[]
[{\bq}] &= (0.0, -1.5708, 1.8675) \pm 0.01~\text{rads}\\
\end{aligned}
\end{equation}
and with the joint velocities $[\dot{\bq}] = [\dot{\bq}_{lim}]$
the inner approximations -- associated polytope with $r=0.9136$, largest n-cube centered at the origin with $r=0.4748$~m/s, and largest n-ball centered at the origin with $r=0.6657$~m/s approximations of $\Omega_{\forall\exists}\left({\bJ}([{\bq}]), [\dot{\bq}]\right)$ -- are given in Figure~\ref{fig:tolsetJqd_xd_plot}.

Analysis of the manipulator's velocity capabilities is performed by bisecting the joint configuration space (given in terms of an interval vector $[\bq]$) into a set of $k$ sub-intervals $[\bq_1],\ldots,[\bq_k]$ at a desired resolution.
Each sub-interval is used to evaluate an inner approximation of the corresponding n-ball centered at the origin, thus the continuous configuration space is completely explored without sampling.
Plots of the values of $r$ throughout the configuration space (see Figure~\ref{fig:tolsetJqd_xd_plot}) can be quickly generated and due to the independence of each sub-interval the performance can benefit significantly from parallel computing.
Furthermore, if given a desired velocity capability, a branch-and-bound loop can iteratively refine the joint configuration space $[\bq]$ finding all satisfying configurations for a desired resolution. This has applications for trajectory planning in the redundant configuration space, where feasibility requires a connected path of satisfying configurations and planning involves selection of an optimal path~\cite{Jaulin2001}. Note that the use of outer approximations of the manipulator's velocity capabilities, which are outside of the scope of this current work, can provide a complementary test allowing to quickly remove non-satisfying configurations from the branch-and-bound loop.

% \vp{would it be worth saying a word about how this compares to measures such as manipulability ?}
Considering the relationship between the velocity kinematics model~\eqref{eq:velocity_problem} and manipulability, the largest n-cube and n-ball inner approximations can be used to rapidly evaluate manipulability measures that correspond to given sets of robot states.

 \begin{figure}[t]
    \centering
    \includegraphics[width=0.4\textwidth]{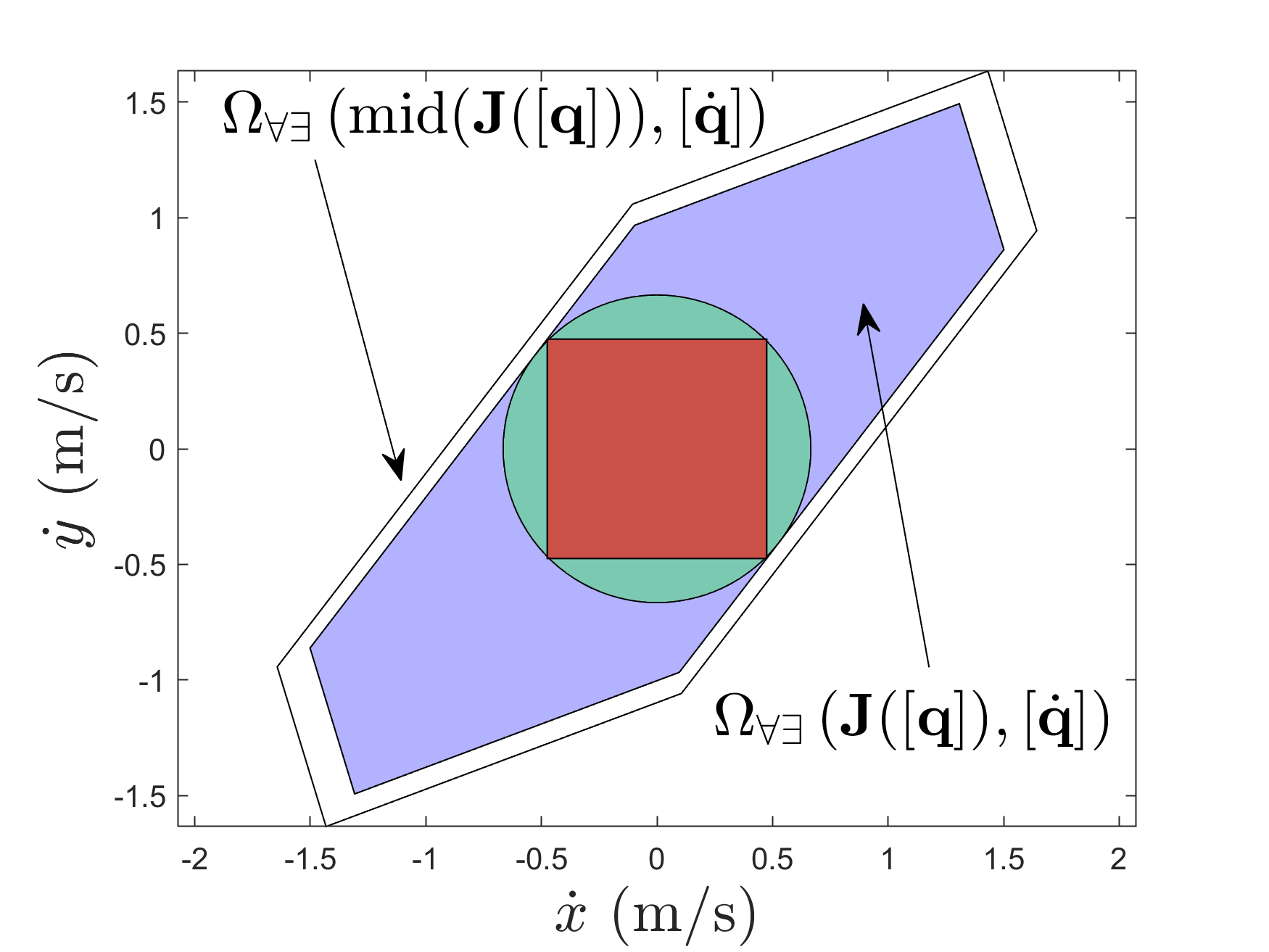}
    \includegraphics[width=0.4\textwidth]{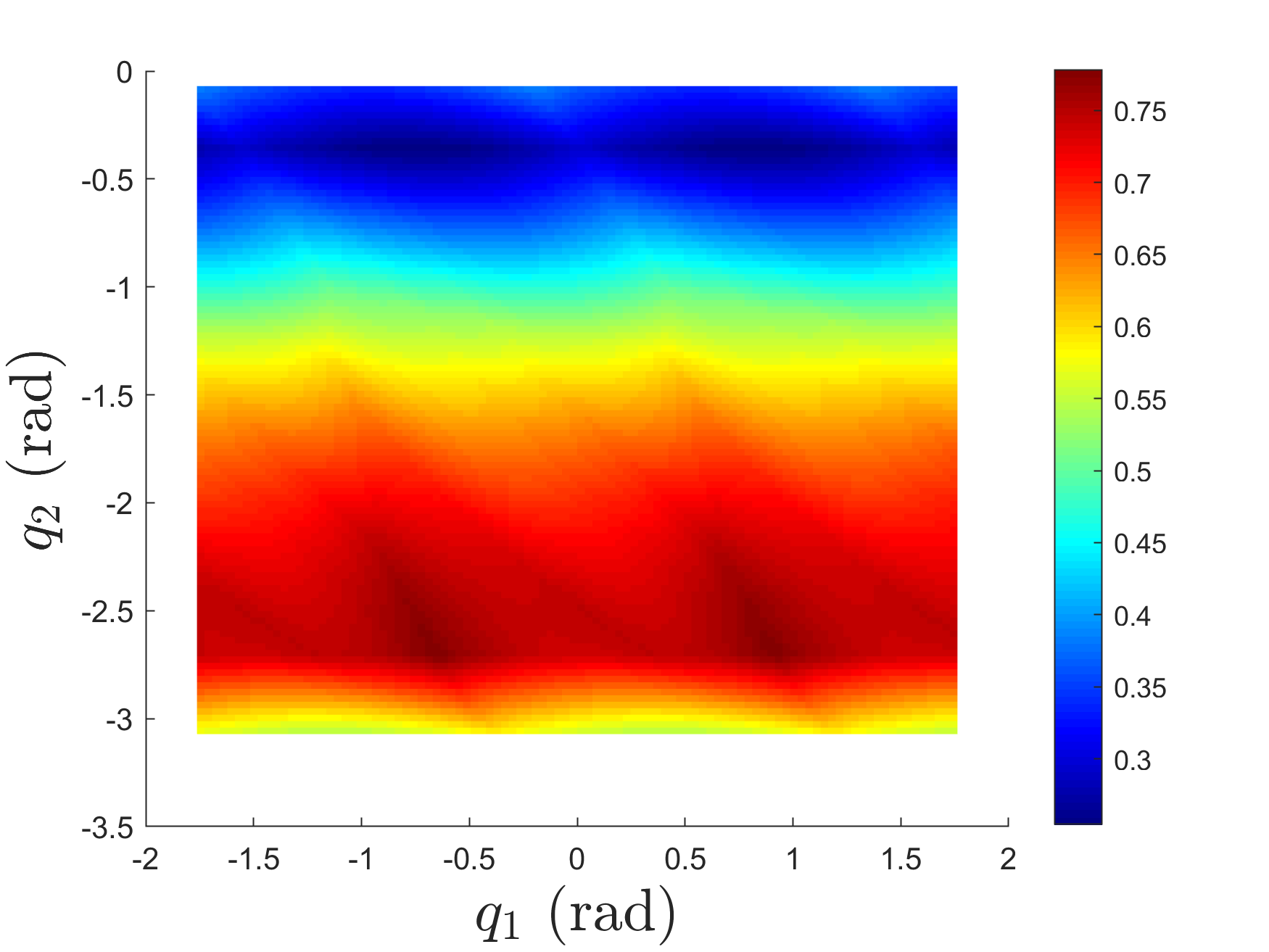}
    \caption{(Top) The associated polytope, largest n-cube centered at the origin, and largest n-ball centered at the origin for the example velocity problem $\Omega_{\forall\exists}\left({\bJ}([{\bq}]), [\dot{\bq}]\right)$. (Bottom) A plot of the values of $r$ throughout the configuration space (with $[{q}_3]=1.8675 \pm 0.01~\text{rads}$) for the largest n-ball centered at the origin.}
    \label{fig:tolsetJqd_xd_plot}
\end{figure}

\subsection{Evaluating the Kinetostatic Problem}
Another necessary procedure when analyzing the performance of a robotic manipulator is to evaluate the possible end-effector wrenches through the mapping of known joint torque/force limits with~\eqref{eq:wrench_problem}.
The kinetostatic problem is of the class $\tolsetb$ and the associated inner approximations of section~\ref{sec:tolsetb} may be used to evaluate the possible end-effector wrenches of the manipulator at a given configuration ${\bq}$.
For a set of configurations $[{\bq}]$, the kinetostatic problem is of the class $\tolsetAb$ and the associated inner approximations of section~\ref{sec:tolsetAb} may be used.

At the same configurations~\eqref{eq:config}
and with the joint torques $[\boldsymbol{\tau}] = [\boldsymbol{\tau}_{lim}]$,
the inner approximations -- associated polytope, largest n-cube centered at the origin with $r=67.3479$~N, and largest n-ball centered at the origin with $r=85.1640$~N approximations of $\Sigma_{\forall\exists}\left({\bJ}([{\bq}])^T, [\boldsymbol{\tau}]\right)$ -- are given in Figure~\ref{fig:tolsetJtau_f_plot_3link}. In addition, the joint configuration space is bisected into sub-intervals and inner approximations of the corresponding n-ball centered at the origin allow to analyze the manipulator's force capabilities.

\begin{figure}[t]
    \centering
    \includegraphics[width=0.4\textwidth]{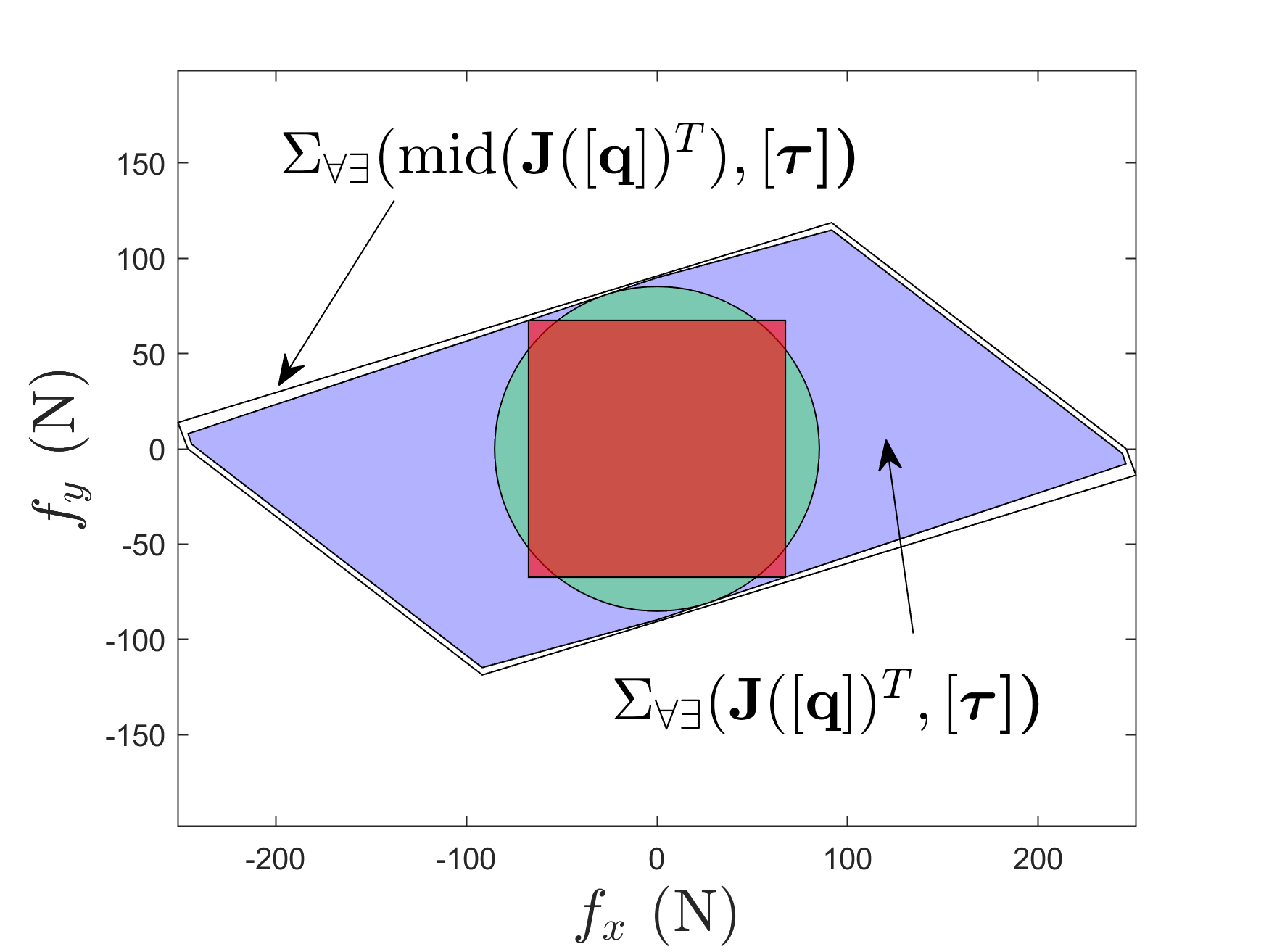}
    \includegraphics[width=0.4\textwidth]{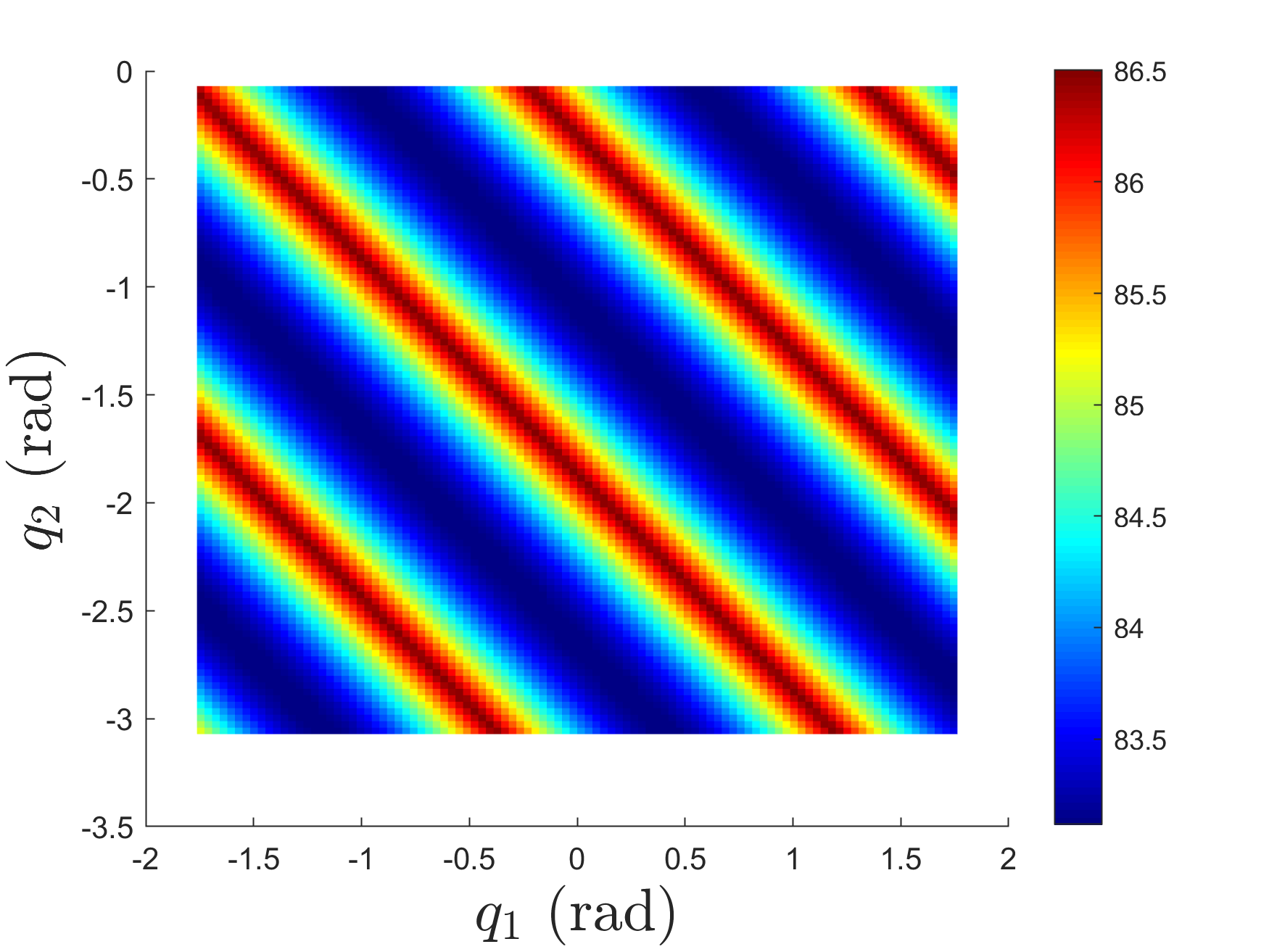}
    \caption{(Top) The associated polytope, largest n-cube centered at the origin, and largest n-ball centered at the origin for the example kinetostatic problem $\Sigma_{\forall\exists}\left({\bJ}([{\bq}])^T, [\boldsymbol{\tau}]\right)$. (Bottom) A plot of the values of $r$ throughout the configuration space (with $[{q}_3]=1.8675 \pm 0.01~\text{rads}$) for the largest n-ball centered at the origin.}
    \label{fig:tolsetJtau_f_plot_3link}
\end{figure}

\subsection{Evaluating the Forward Acceleration Problem}
% \vp{From a control point of view, what make sense is to evaluate the acceleration capabilities in the current state q, qdot and for given acceleration limits. So while I see why one could consider an interval on q and qdot, to me it will be easier for the reader to understand our point if we only consider bounds on qddot. Also xddot1 shoudl be considered in the current state and not account for bounds on qdot.}

The local operational acceleration capabilities can be determined from the forward acceleration model~\eqref{eq:acceleration_problem}. The problem can be split into two interval linear systems of equations where each system needs to be solved to give a corresponding set of operational accelerations (see Tables~\ref{tab:usecases_tolsetx}). %and \ref{tab:usecases_tolsetAx}).
The Minkowski sum of the two sets then gives the local operational acceleration capabilities. Inner polytope approximations of the two sets of class $\tolsetx$ or $\tolsetAx$ are obtained from:
\begin{equation}\label{eq:forward_accel_problem}
    \begin{aligned}
     \ddot{\bx}_1 &= \dot{\bJ}({\bq},\dot{\bq}) \dot{\bq}\\
     \ddot{\bx}_2 &= {\bJ}({\bq}) \ddot{\bq}\\
    \end{aligned}
\end{equation}
Since each inner polytope approximation is convex, the Minkowski sum of the two sets is easily computed as the convex hull of all combinations of polytope vertices. The inner approximations~\eqref{eq:tolsetx_cube_r} and~\eqref{eq:tolsetx_ball_r} can then be applied.

Let $\dot{\bJ}({\bq},\dot{\bq})$ and ${\bJ}({\bq})$ be evaluated at the same configurations~\eqref{eq:config}
and with the joint velocities
\begin{equation}\label{eq:velconfig}
[\dot{\bq}] = (1.0,1.0,1.0) \pm 0.01~\text{rads/s}\\
\end{equation}
% For the purpose of this example, the manufacturer's joint accelerations limits $[\ddot{\bq}_{lim}]$ are taken as constants, although realistically these limits are state dependant.
% First, neglecting the design parameter and configuration uncertainties by selecting the midpoints of~\eqref{eq:parametervalues}, \eqref{eq:config} and~\eqref{eq:velconfig} (i.e., ${\bq} = \cmid{[{\bq}]}$ and $\dot{\bq} = \cmid{[\dot{\bq}]}$), the inner polytope approximation of $\Omega_{\forall\exists}({\bJ}({\bq}), [\ddot{\bq}_{lim}])$, with $r=1$ since ${\bJ}({\bq})$ is not an interval matrix,
% % \vp{What do you mean by r = 1? If it is a result why not plotting it in the figure?}
% and point
% % \vp{rather small on the figure.} {\jp Addressed in todo}
% corresponding to $\Omega_{\forall\exists}(\dot{\bJ}({\bq},\dot{\bq}), [\dot{\bq}_{lim}])$ are computed. %depicted in Figure~\ref{fig:tolsetJdqd_plus_Jqqdd_xdd_plot_3link_combined}.
% Then the inner approximations of the Minkowski sum of the two sets are depicted in Figure~\ref{fig:tolsetJdqd_plus_Jqqdd_xdd_plot_3link_combined}: largest n-cube centered at the origin with $r=2.0032$~m/s$^2$, and largest n-ball centered at the origin with $r=2.8088$~m/s$^2$.
As more uncertainties or variabilities in the system are considered, the greater the widths of the intervals in the matrices $[\bA]$ become and therefore the greater the widths of the vectors $[\bx]$ must be to in order to determine an inner polytope approximation with $r$ greater than zero.
Considering the joint state limits $[\dot{\bq}_{lim}]$ and $[\ddot{\bq}_{lim}]$, the resulting inner polytope approximations of $\Omega_{\forall\exists}(\dot{\bJ}([{\bq}],[\dot{\bq}]), [\dot{\bq}_{lim}])$ and $\Omega_{\forall\exists}({\bJ}([{\bq}]), [\ddot{\bq}_{lim}])$ are $r=0.8446$ and $r=0.9237$, respectively (see Figure~\ref{fig:forward_accel_init} top).
% The inclusion of the uncertainties, specifically on $[\dot{\bq}]$, generates the polytope $\Omega_{\forall\exists}(\dot{\bJ}([{\bq}],[\dot{\bq}]), [\dot{\bq}_{lim}])$.
% \vp{why is there still a point plotted?}. {\jp In fact it is a polytope that is too small to see properly. Addressed in todo.}
The inner approximations of the operational acceleration capabilities is Minkowski sum of the two sets (see Figure~\ref{fig:forward_accel_init} bottom) with: largest n-cube centered at the origin with $r=5.7386$~m/s$^2$, and largest n-ball centered at the origin with $r=7.7526$~m/s$^2$.
% \vp{in intro I insist on the fact limits in temrs of acceleration are not constant. May ba say a word here on the fact that they are considered constant for the sake of the example. In fact I wonder whether it make sense to consider the case xdd = J qdd + Jd qd as qdd is actually not directly bounded.}
% As expected, by considering the uncertainties, the local operational acceleration capabilities described by the associated interval linear system of equations are slightly reduced compared to the capabilities which neglect the uncertainties.
% \vp{would not it make sense to superpose fig. 9 and 10. for the comparison?}. {\jp Address in todo}
While uncertainties may be largely ignored from a control point of view, accounting for these uncertainties is the only way to truly certify the capabilities of the system.

\begin{figure}[t]
    \centering
    \includegraphics[width=0.45\textwidth]{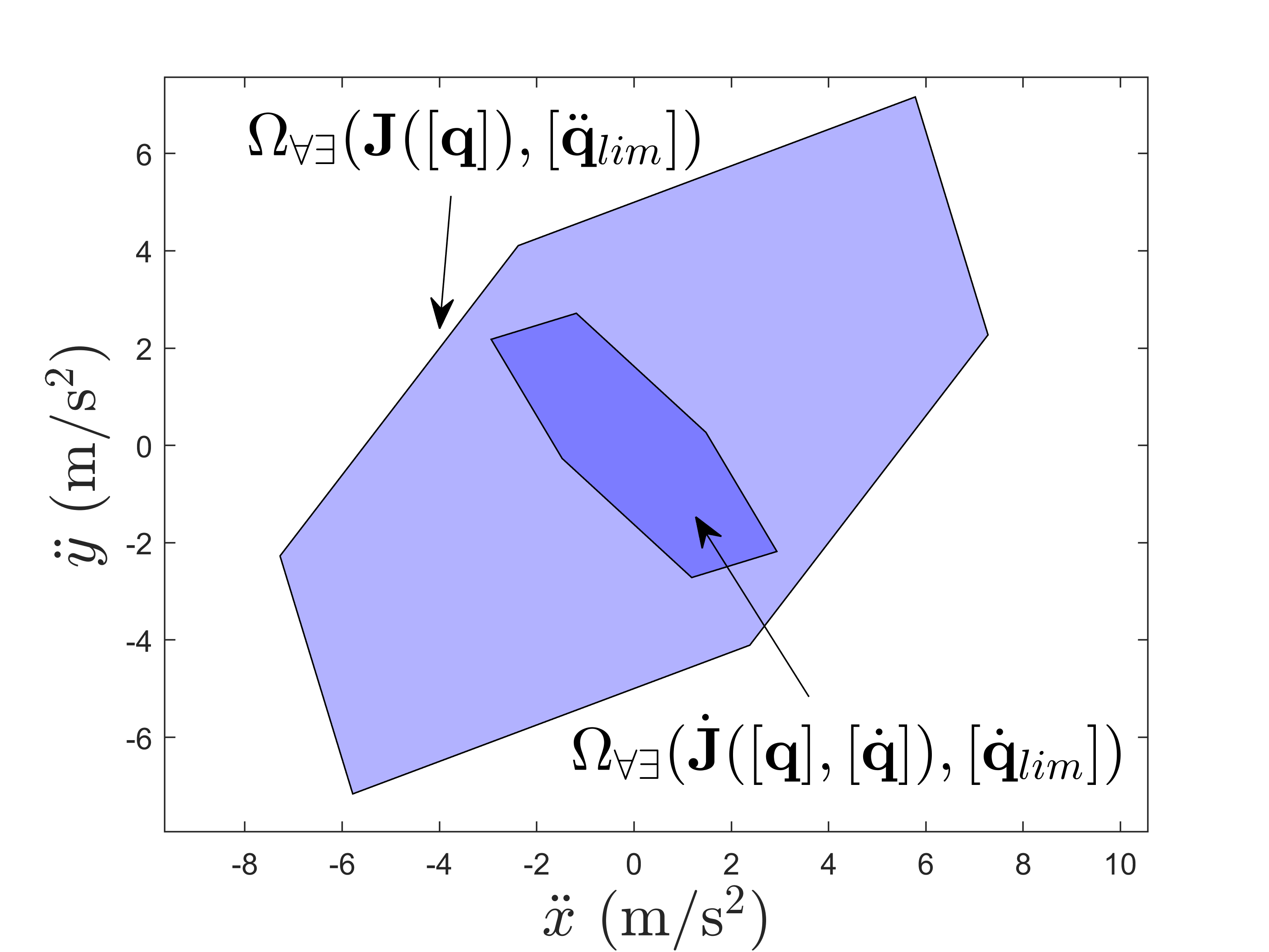}
    \includegraphics[width=0.45\textwidth]{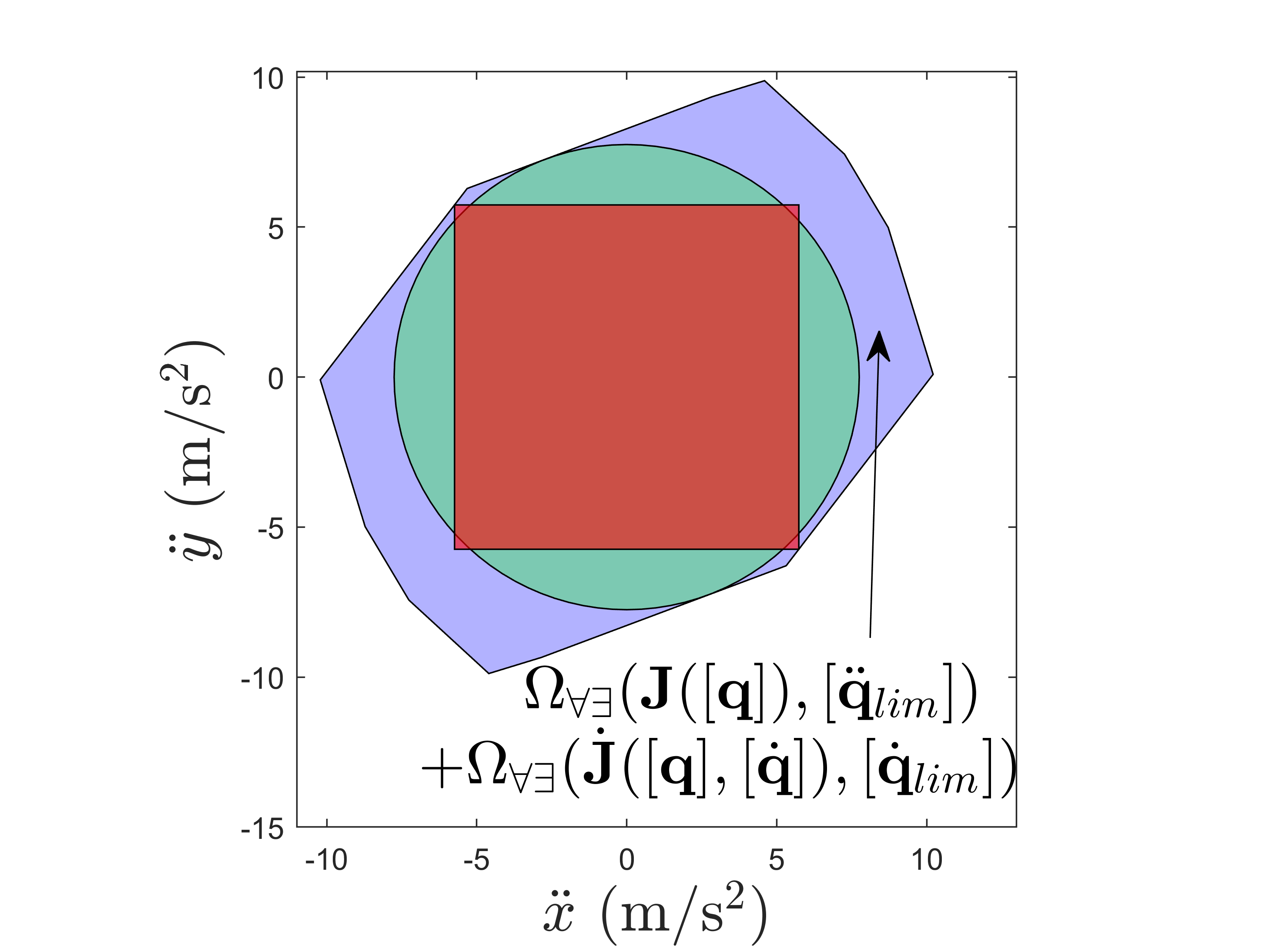}
    \caption{(Top) The associated inner polytope approximations for $\Omega_{\forall\exists}(\dot{\bJ}({\bq},\dot{\bq}), [\dot{\bq}_{lim}])$ and $\Omega_{\forall\exists}({\bJ}({\bq}), [\ddot{\bq}_{lim}])$. (Bottom) The associated inner approximations for $\Omega_{\forall\exists}(\dot{\bJ}({\bq},\dot{\bq}), [\dot{\bq}_{lim}]) + \Omega_{\forall\exists}({\bJ}({\bq}), [\ddot{\bq}_{lim}])$.% The dashed line shows a direct comparison between the use of $\ddot{\bq} = [\ddot{\bq}_{lim}]$ and $\ddot{\bq} = {\bf M}([{\bq}_f])^{-1}[\boldsymbol{\tau}_{e}]$.
    % The unfilled polytope shows the future operational accelerations computed from~\eqref{eq:qdd_sub} \vp{I would either remove it or say something in the text about it}.
    }
    % {\mh[MH: I do not understand these figures on the top.]} {\jp Address in todo}
    \label{fig:forward_accel_init}
\end{figure}

\subsection{Evaluating Future Dynamic Capabilities}
% \vp{we should explicitly compare the bounds on xddot obtained with constant bounds on qddot (previous section) to the one obtained with bounds on the torque (and we should do that for two or three values of the state). This would be a great illustration of the need for accounting for the real limits of the system. We have also to discuss about vg. It should be treated the same way as Jdot qdot.}

% \vp{here again I would not consider limits on qdd and qddd but rather limits on q, qd, tau and taud only.}
Let the current state of a robot be given by $\{\bq,\dot{\bq},\boldsymbol{\tau}\}$ and let the joint state limits be given by the intervals $\{[\bq_{lim}],[\dot{\bq}_{lim}],[\boldsymbol{\tau}_{lim}]\}$ and the control limit be given by the interval $[\dot{\boldsymbol{\tau}}_{lim}]$.
Consider a prediction period interval $[t] = [0,t_{max}]$.
The future joint torques $[\boldsymbol{\tau}_{f}]$ are bounded by
\begin{equation}
    [\boldsymbol{\tau}_{f}] = \left(\boldsymbol{\tau} +  [\dot{\boldsymbol{\tau}}_{lim}] [t]\right)  \cap [\boldsymbol{\tau}_{lim}]
\end{equation}
The future joint accelerations $[\ddot{\bq}_{f}]$ are estimated from joint torque at the current state by
\begin{equation}
    [\ddot{\bq}_{f}] = {\bf M}({\bq})^{-1}\left([\boldsymbol{\tau}_{f}] - \bv({\bq},\dot{{\bq}}) -{\bg}({\bq})\right).
\end{equation}
% where the intersection ensures that $[\ddot{\bq}_{f}] \subseteq [\ddot{\bq}_{lim}]$.

%\dd{Is it the vector C in place of V}

The future joint velocities $[\dot{\bq}_{f}]$ are given by
\begin{equation}
    [\dot{\bq}_{f}] = \left(\dot{\bq} +  [\ddot{\bq}_{{f}}] [t]\right)  \cap [\dot{\bq}_{lim}].
\end{equation}
% where the intersection ensures that $[\dot{\bq}_{f}] \subseteq [\dot{\bq}_{lim}]$.
The future joint configurations $[\boldsymbol{{{\bq}}}_{f}]$ are given by
\begin{equation}
\begin{aligned}[]
[\boldsymbol{{{\bq}}}_{f}] = \left(\boldsymbol{{{\bq}}} + [\dot{\bq}_{f}] [t] + \frac{1}{2}[\ddot{\bq}_{{f}}] [t] ^ 2 \right) \cap [\boldsymbol{{{\bq}}}_{lim}].
\end{aligned}
\end{equation}
% where the intersection ensures that $[\boldsymbol{{{\bq}}}_{f}] \subseteq [\boldsymbol{{{\bq}}}_{lim}]$.
The joint state limit intersections ensure that the limits are not exceeded during the prediction of the future set of states.

% Certain manufacturers, such as Franka Emika for their Panda collaborative robot, provide joint state limits for $[\dddot{\bq}_{lim}]$ and $[\dot{\boldsymbol{\tau}}_{lim}]$. If these joint state limits or others are not provided, then the worst-case conditions (\emph{e.g.}, $[\ddot{\bq}_{f}] = [\ddot{\bq}_{lim}]$ and $[\boldsymbol{\tau}_{f}]=[\boldsymbol{\tau}_{lim}]$) can be used to ensure reliability of the computations.

It is clear that joint acceleration capabilities are state dependent and therefore if constant joint acceleration limits are enforced, the controller may either require infeasible accelerations from the system or sub-optimally use its actual capabilities.
Furthermore, due to the noise associated with the estimation of the current $\ddot{\bq}$, enforcing constant joint acceleration limits is difficult in practice.
Alternatively, the use of joint torques to evaluate joint acceleration capabilities can provide a much better estimate of the capabilities of the system.

The joint acceleration capabilities achievable during the prediction period $[t]$ can be determined from the configuration-space dynamic model~\eqref{eq:configuration_dynamic_problem}.
However, the right hand side of~\eqref{eq:configuration_dynamic_problem} should not be overestimated. Therefore, let the effective joint torques be given by $[\boldsymbol{\tau}_{e}]= [\underline{\boldsymbol{\tau}_{f}} + \overline{\bf cg}, \overline{\boldsymbol{\tau}_{f}} + \underline{\bf cg}]$, where $[{\bf cg}] = -{\bv}([{\bq}_f],[\dot{\bq}_f]) -{\bg}([{\bq}_f])$.
Inner approximations of the joint acceleration capabilities are then computed from
\begin{equation}
    {\bf M}([{\bq}_f])\ddot{\bq} = [\boldsymbol{\tau}_{e}].
\end{equation}
The problem is of class $\tolsetAb$ and the associated inner approximations of section~\ref{sec:tolsetAb} may be used.

As an example, let the current states be given by \eqref{eq:config} and \eqref{eq:velconfig} and
% \begin{equation}\label{eq:states}
%     \begin{aligned}[]
%         {\bq} &= (0.0, -1.5708, 1.8675)~\text{rad}\\
%         \dot{\bq} &= (1.0,1.0,1.0)~\text{rad/s}\\
%         % \ddot{\bq} &= (0.5,0.5,0.5)~\text{rad/s}^2\\
%         \boldsymbol{\tau} &= (18.0,1.0,2.0)~\text{Nm}\\
%     \end{aligned}
% \end{equation}
\begin{equation}\label{eq:states}
    \boldsymbol{\tau} = (18.0,1.0,2.0)~\pm 0.01~\text{Nm}\\
\end{equation}
% The design parameter uncertainties are neglected by selecting the midpoints of~\eqref{eq:parametervalues}.
Selecting the time $[t] = [0,0.01]$~s,
the corresponding joint acceleration inner approximation polytope, the largest n-cube centered at the origin with $r=2.3382~\text{rad/s}^2$, and the largest n-ball centered at the origin with $r=3.2315~\text{rad/s}^2$ are shown in Figure~\ref{fig:tolsetMtaueff_qdd_plot_3link_nouncertain}.
For visibility reasons only, the polytope is bounded by the manufacturer's limits $[\ddot{\bq}_{lim}]$.

The inner polytope approximation of the joint acceleration capabilities represents a set of valid joint accelerations which better describe the future operational accelerations during a prediction period $[t]$. This can be used in the form of inequality constraints inside the robot controller, augmenting the manufacturer's default limits, to ensure that the robot's capabilities are properly considered to select valid control actions.

\begin{figure}[t]
    \centering
    \includegraphics[width=0.5\textwidth]{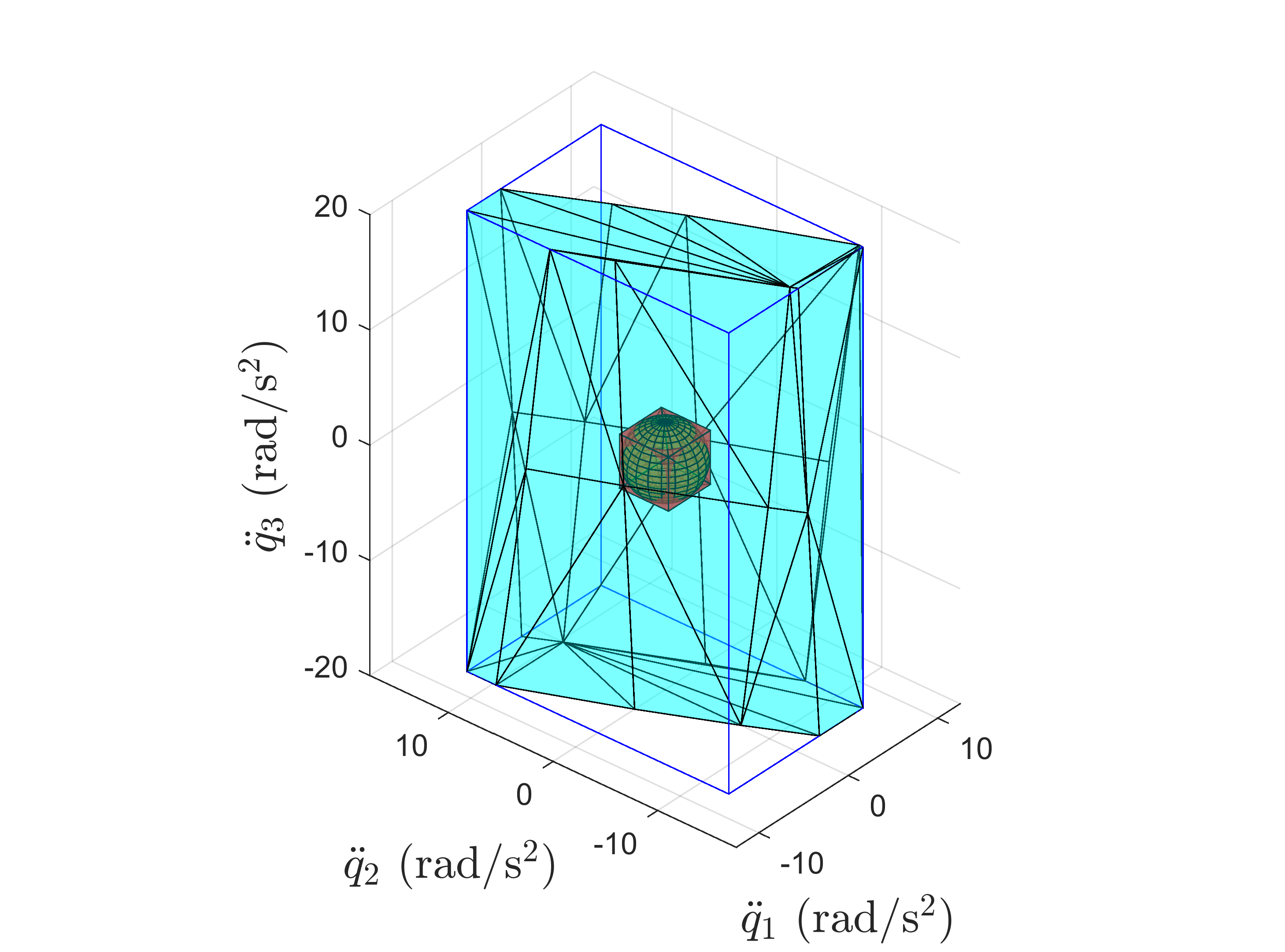}
    % \caption{(Top) Local joint acceleration capability inner approximations. (Bottom) Local dynamic wrench capability inner approximations.}
    \caption{Future joint acceleration capability inner approximations during a prediction period $[t]$. For visibility, the polytope is bounded by the manufacturer's limits $[\ddot{\bq}_{lim}]$.
    % (Bottom) Future operational acceleration capability inner approximation.
    }
    \label{fig:tolsetMtaueff_qdd_plot_3link_nouncertain}
\end{figure}

\section{Conclusion}
Efficient set-based approaches for computing reliable inner approximations of a robot's common capabilities over sets of joint states and with design parameter uncertainties are presented in this work. These approaches allow to not only consider the temporal evolution of the system over a small time horizon when computing the localized capabilities, but also allow to manage uncertainties to handle an imprecise/variable system describing a family of robotic manipulators.

The use of set-based approaches to evaluate capabilities provides many benefits over conventional approaches. One key benefit is the computation of certified inner approximations of the robot's capabilities, allowing these approaches to be used in offline and online scenarios to improve the safety of robotic systems. The ability to automatically handle design parameter uncertainties and other system uncertainties affecting performance (e.g., inertial parameter errors~\cite{giftsun:hal-01533136}) allows for the use of simpler imprecise/variable models that accurately describe a complex system (\emph{e.g.}, modeling flexure-jointed mechanisms~\cite{10.1115/1.3042151}).
Also, the ability to consider sets of joint states allows to locally evaluate capabilities over a small time horizon, such that the robot's performance in the near future can be better understood. The sets of joint states may also be used to analyze the capabilities along continuous trajectories via the use of branch-and-bound methods, as opposed to the standard sampling-based approach, allowing to properly understand worst-case conditions which can be used to improve the overall performance of the robot.

The low computation costs of many of the proposed inner approximations allow for real-time applications, enabling their use for rapid design analysis as well as their potential integration directly in to the robot controller. Therefore, a future application of this work may be the development of a certifiably safe robot controller which will use a properly formulated imprecise/variable system to safely estimate the robot's true capabilities over a given time horizon. These capabilities may then be used directly in the controller to ensure that each controller action is achievable. Furthermore, by computing and incorporating the robot's deceleration capabilities into the controller, it is be possible to enforce strict braking criteria to guarantee that collaborative systems are safe when in close proximity to a human.

% % IJRR
% \begin{acks}
% This work was supported by the Czech Science Foundation Grant P403-18-04735S.
% \end{acks}

\bibliographystyle{IEEEtran}
\bibliography{IEEEabrv,bib}

\begin{thebibliography}{10}
\providecommand{\url}[1]{#1}
\csname url@samestyle\endcsname
\providecommand{\newblock}{\relax}
\providecommand{\bibinfo}[2]{#2}
\providecommand{\BIBentrySTDinterwordspacing}{\spaceskip=0pt\relax}
\providecommand{\BIBentryALTinterwordstretchfactor}{4}
\providecommand{\BIBentryALTinterwordspacing}{\spaceskip=\fontdimen2\font plus
\BIBentryALTinterwordstretchfactor\fontdimen3\font minus
  \fontdimen4\font\relax}
\providecommand{\BIBforeignlanguage}[2]{{%
\expandafter\ifx\csname l@#1\endcsname\relax
\typeout{** WARNING: IEEEtran.bst: No hyphenation pattern has been}%
\typeout{** loaded for the language `#1'. Using the pattern for}%
\typeout{** the default language instead.}%
\else
\language=\csname l@#1\endcsname
\fi
#2}}
\providecommand{\BIBdecl}{\relax}
\BIBdecl

\bibitem{Ibanez2017}
A.~Ibanez, P.~Bidaud, and V.~Padois, \emph{Optimization-Based Control
  Approaches to Humanoid Balancing}.\hskip 1em plus 0.5em minus 0.4em\relax
  Springer Netherlands, 2017, pp. 1--27.

\bibitem{tonneau2018efficient}
S.~Tonneau, A.~Del~Prete, J.~Pettr{\'e}, C.~Park, D.~Manocha, and N.~Mansard,
  ``An efficient acyclic contact planner for multiped robots,'' \emph{IEEE
  Transactions on Robotics}, vol.~34, no.~3, pp. 586--601, 2018.

\bibitem{mastalli2019crocoddyl}
C.~Mastalli, R.~Budhiraja, W.~Merkt, G.~Saurel, B.~Hammoud, M.~Naveau,
  J.~Carpentier, L.~Righetti, S.~Vijayakumar, and N.~Mansard, ``Crocoddyl: An
  efficient and versatile framework for multi-contact optimal control,'' in
  \emph{Proceedings of the IEEE RAS International Conference in Robotics and
  Automation}, Paris, France, May 2020.

\bibitem{delprete2018joint}
A.~D. {Prete}, ``Joint position and velocity bounds in discrete-time
  acceleration/torque control of robot manipulators,'' \emph{IEEE Robotics and
  Automation Letters}, vol.~3, no.~1, pp. 281--288, 2018.

\bibitem{rubrecht2012motion}
S.~Rubrecht, V.~Padois, P.~Bidaud, M.~De~Broissia, and M.~D.~S. Simoes,
  ``Motion safety and constraints compatibility for multibody robots,''
  \emph{Autonomous Robots}, vol.~32, no.~3, pp. 333--349, 2012.

\bibitem{decre2009extending}
W.~{Decre}, R.~{Smits}, H.~{Bruyninckx}, and J.~{De Schutter}, ``Extending
  itasc to support inequality constraints and non-instantaneous task
  specification,'' in \emph{2009 IEEE International Conference on Robotics and
  Automation}, 2009, pp. 964--971.

\bibitem{8723340}
J.~G. N.~D. {Carvalho Filho}, E.~A.~N. {Carvalho}, L.~{Molina}, and E.~O.
  {Freire}, ``The impact of parametric uncertainties on mobile robots
  velocities and pose estimation,'' \emph{IEEE Access}, vol.~7, pp.
  69\,070--69\,086, 2019.

\bibitem{Moore:107064}
R.~E. Moore, \emph{{Interval analysis}}, ser. Prentice-Hall series in automatic
  computation.\hskip 1em plus 0.5em minus 0.4em\relax Englewood Cliffs, NJ:
  Prentice-Hall, 1966.

\bibitem{jaulin:hal-00845131}
L.~Jaulin, M.~Kieffer, O.~Didrit, and E.~Walter, \emph{{Applied Interval
  Analysis with Examples in Parameter and State Estimation, Robust Control and
  Robotics}}.\hskip 1em plus 0.5em minus 0.4em\relax Springer London, Aug.
  2001.

\bibitem{MooKea2009}
R.~E. Moore, R.~B. Kearfott, and M.~J. Cloud, \emph{Introduction to Interval
  Analysis}.\hskip 1em plus 0.5em minus 0.4em\relax Philadelphia, PA: SIAM,
  2009.

\bibitem{PopHla2013}
E.~D. Popova and M.~Hlad{\'\i}k, ``Outer enclosures to the parametric {AE}
  solution set,'' \emph{Soft Comput.}, vol.~17, no.~8, pp. 1403--1414, 2013.

\bibitem{Shary2001}
S.~P. Shary, ``Interval gauss-seidel method for generalized solution sets to
  interval linear systems,'' \emph{Reliable Computing}, vol.~7, no.~2, pp.
  141--155, Apr 2001.

\bibitem{Kol2004c}
L.~V. Kolev, ``A method for outer interval solution of linear parametric
  systems,'' \emph{Reliab. Comput.}, vol.~10, no.~3, pp. 227--239, 2004.

\bibitem{PopKra2007}
E.~D. Popova and W.~Kr{\"a}mer, ``Inner and outer bounds for the solution set
  of parametric linear systems,'' \emph{J. Comput. Appl. Math.}, vol. 199,
  no.~2, pp. 310--316, 2007.

\bibitem{Neumaier1993}
A.~Neumaier, ``The wrapping effect, ellipsoid arithmetic, stability and
  confidence regions,'' in \emph{Validation Numerics: Theory and Applications},
  R.~Albrecht, G.~Alefeld, and H.~J. Stetter, Eds.\hskip 1em plus 0.5em minus
  0.4em\relax Vienna: Springer, 1993, pp. 175--190.

\bibitem{Nedialkov2004}
N.~Nedialkov, V.~Kreinovich, and S.~Starks, ``Interval arithmetic, affine
  arithmetic, {Taylor} series methods: Why, what next?'' \emph{Numerical
  Algorithms}, vol.~37, pp. 325--336, 12 2004.

\bibitem{merlet:inria-00362431}
J.-P. Merlet, ``{Interval analysis for Certified Numerical Solution of Problems
  in Robotics},'' \emph{{International Journal of Applied Mathematics and
  Computer Science}}, 2009.

\bibitem{10.1115/1.3042151}
D.~Oetomo, D.~Daney, B.~Shirinzadeh, and J.-P. Merlet, ``{An Interval-Based
  Method for Workspace Analysis of Planar Flexure-Jointed Mechanism},''
  \emph{Journal of Mechanical Design}, vol. 131, no.~1, 12 2008.

\bibitem{5657268}
M.~{Gouttefarde}, D.~{Daney}, and J.-P. Merlet, ``Interval-analysis-based
  determination of the wrench-feasible workspace of parallel cable-driven
  robots,'' \emph{IEEE Transactions on Robotics}, vol.~27, no.~1, pp. 1--13,
  Feb 2011.

\bibitem{doi:10.1139/tcsme-2016-0012}
J.~K. Pickard and J.~A. Carretero, ``An interval analysis method for wrench
  workspace determination of parallel manipulator architectures,''
  \emph{Transactions of the Canadian Society for Mechanical Engineering},
  vol.~40, no.~2, pp. 139--154, 2016.

\bibitem{Merlet2008}
J.-P. Merlet and D.~Daney, ``Appropriate design of parallel manipulators,'' in
  \emph{Smart Devices and Machines for Advanced Manufacturing}, L.~Wang and
  J.~Xi, Eds.\hskip 1em plus 0.5em minus 0.4em\relax London: Springer, 2008,
  pp. 1--25.

\bibitem{farzanehkaloorazi_masouleh_caro_2017}
M.~H.~F. Kaloorazi, M.~T. Masouleh, and S.~Caro, ``Collision-free workspace of
  parallel mechanisms based on an interval analysis approach,''
  \emph{Robotica}, vol.~35, no.~8, p. 1747–1760, 2017.

\bibitem{10.1007/978-3-642-14743-2_13}
J.-P. Merlet, ``Interval analysis and robotics,'' in \emph{Robotics Research},
  M.~Kaneko and Y.~Nakamura, Eds.\hskip 1em plus 0.5em minus 0.4em\relax
  Berlin, Heidelberg: Springer, 2011, pp. 147--156.

\bibitem{10.1007/978-1-4020-4941-5_5}
J.-P. Merlet and P.~Donelan, ``On the regularity of the inverse jacobian of
  parallel robots,'' in \emph{Advances in Robot Kinematics},
  J.~Lenar{\v{c}}i{\v{c}} and B.~Roth, Eds.\hskip 1em plus 0.5em minus
  0.4em\relax Dordrecht: Springer Netherlands, 2006, pp. 41--48.

\bibitem{Bouchard2009}
S.~Bouchard, C.~M. Gosselin, and B.~Moore, ``On the ability of a cable-driven
  robot to generate a prescribed set of wrenches,'' \emph{ASME. J. Mechanisms
  Robotics}, vol.~2, no.~1, 2009.

\bibitem{carretero-gosselin-2010-cable}
J.~A. Carretero and C.~M. Gosselin, ``Wrench capabilities of cable-driven
  parallel mechanisms using wrench polytopes,'' in \emph{Proceeding of the 2010
  IFToMM Symposium on Mechanism Design for Robotics}, Universidad Panamericana,
  Mexico City, Mexico, September 28--30 2010.

\bibitem{Gouttefarde2010475482}
M.~Gouttefarde and S.~Krut, ``Characterization of parallel manipulator
  available wrench set facets,'' in \emph{Advances in Robot Kinematics: Motion
  in Man and Machine}, J.~Lenar{\v{c}}i{\v{c}} and M.~M. Stanisic, Eds.\hskip
  1em plus 0.5em minus 0.4em\relax Dordrecht: Springer Netherlands, 2010, pp.
  475--482.

\bibitem{Shary2002}
S.~P. Shary, ``A new technique in systems analysis under interval uncertainty
  and ambiguity,'' \emph{Reliable Computing}, vol.~8, no.~5, pp. 321--418, Oct
  2002.

\bibitem{Zie2004}
G.~M. Ziegler, \emph{Lectures on Polytopes}, ser. Graduate Texts in
  Mathematics.\hskip 1em plus 0.5em minus 0.4em\relax Berlin: Springer, 1995,
  vol. 152.

\bibitem{WenChen2010}
R.~E. Wendell and W.~Chen, ``Tolerance sensitivity analysis: {Thirty} years
  later,'' \emph{Croatian Oper. Res. Rev.}, vol.~1, pp. 12--21, 2010.

\bibitem{Hla2011c}
M.~Hlad\'{\i}k, ``Tolerance analysis in linear systems and linear
  programming,'' \emph{Optim. Methods Softw.}, vol.~26, no.~3, pp. 381--396,
  2011.

\bibitem{Hla2013b}
------, ``Weak and strong solvability of interval linear systems of equations
  and inequalities,'' \emph{Linear Algebra Appl.}, vol. 438, no.~11, pp.
  4156--4165, 2013.

\bibitem{Roh1998b}
J.~Rohn, ``Linear programming with inexact data is {NP}-hard,'' \emph{ZAMM, Z.
  Angew. Math. Mech.}, vol.~78, no. Supplement 3, pp. S1051--S1052, 1998.

\bibitem{Rohn2006}
------, ``Solvability of systems of interval linear equations and
  inequalities,'' in \emph{Linear Optimization Problems with Inexact Data},
  M.~Fiedler~et al., Ed.\hskip 1em plus 0.5em minus 0.4em\relax Boston, MA:
  Springer, 2006, ch.~2, pp. 35--77.

\bibitem{SHARY199553}
S.~P. Shary, ``Solving the linear interval tolerance problem,''
  \emph{Mathematics and Computers in Simulation}, vol.~39, no.~1, pp. 53 -- 85,
  1995.

\bibitem{Pop2013a}
E.~D. Popova, ``Inner estimation of the parametric tolerable solution set,''
  \emph{Comput. Math. Appl.}, vol.~66, no.~9, pp. 1655--1665, 2013.

\bibitem{Rohn1986157158}
J.~Rohn, ``Inner solutions of linear interval systems,'' in \emph{Proceedings
  of the International Symposium on Interval Mathematics on Interval
  Mathematics 1985}.\hskip 1em plus 0.5em minus 0.4em\relax Berlin, Heidelberg:
  Springer, 1986, pp. 157--158.

\bibitem{Chiacchio1996}
P.~{Chiacchio}, Y.~{Bouffard-Vercelli}, and F.~{Pierrot}, ``Evaluation of force
  capabilities for redundant manipulators,'' in \emph{Proceedings of IEEE
  International Conference on Robotics and Automation}, vol.~4, April 1996, pp.
  3520--3525.

\bibitem{skuric2020online}
A.~Skuric, V.~Padois, and D.~Daney, ``On-line force capability evaluation based
  on efficient polytope vertex search,'' 2020.

\bibitem{Pickard2020}
\BIBentryALTinterwordspacing
J.~K. Pickard, ``{Software: Inner Approximation Algorithms for Interval Linear
  Systems of Equations},'' June 2020. [Online]. Available:
  \url{https://figshare.com/articles/Software_Inner_Approximation_Algorithms_for_Interval_Linear_Systems_of_Equations/12472748}
\BIBentrySTDinterwordspacing

\bibitem{PICKARD2019237}
J.~K. Pickard, J.~A. Carretero, and J.-P. Merlet, ``Appropriate analysis of the
  four-bar linkage,'' \emph{Mechanism and Machine Theory}, vol. 139, pp. 237 --
  250, 2019.

\bibitem{Jaulin2001}
L.~Jaulin, ``Path planning using intervals and graphs,'' \emph{Reliable
  Computing}, vol.~7, pp. 1--15, 02 2001.

\bibitem{giftsun:hal-01533136}
N.~Giftsun, A.~Del~Prete, and F.~Lamiraux, ``{Robustness to Inertial Parameter
  Errors for Legged Robots Balancing on Level Ground},'' in
  \emph{{International Conference on Informatics in Control, Automation and
  Robotics (ICINCO 2017)}}, Madrid, Spain, Jul. 2017.

\end{thebibliography}

\begin{IEEEbiography}[{\includegraphics[width=1in,height=1.25in,clip,keepaspectratio]{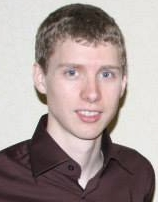}}]{Joshua K. Pickard}
received a B.Sc.E. degree in mechanical engineering with a mechatronics option in 2012, and a Ph.D. in mechanical engineering on uncertainty analysis and mechanism synthesis in 2018 from from the University of New Brunswick, Fredericton, New Brunswick, Canada.
From 2018 to 2020 he was a postdoctoral fellow at Inria Bordeaux Sud-Ouest, Bordeaux, France, with the Auctus Team working on a modeling and analysis framework for imprecise and variable kinematic chains for the study of human motor-variabilities.
Currently, he is a postdoctoral fellow at the University of New Brunswick working on research related to factories of the future and Industry 4.0.

\end{IEEEbiography}
\begin{IEEEbiography}[{\includegraphics[width=1in,height=1.25in,clip,keepaspectratio]{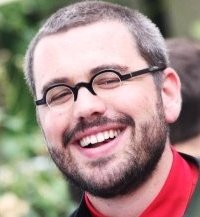}}]{Vincent Padois} is a senior research scientist at Inria, centre Bordeaux Sud-Ouest in the Auctus team (Talence, France). From 2007 to 2020, he was an associate professor of Robotics at Sorbonne Université and ISIR (Paris, France). From 2006 to 2007 he was  a postdoctoral researcher at the Stanford Artiﬁcial Intelligence Laboratory in the group of Prof. O. Khatib (Stanford, USA). He received his Ph.D. degree in Robotics and Automatic Control from INP de Toulouse in 2005. His main research interest is robot control in constrained contexts and dynamic environments, including wheeled mobile manipulators, humanoid robots and collaborative robots. Beyond control, his research activities in collaborative robotics are focused on virtual human models for the ergonomic quantification of the physical assistance of collaborative robots. He is also involved in research activities aiming at bridging the gap between adaptation and decision making techniques and model-based control through the use of machine learning and evolutionary algorithms.
\end{IEEEbiography}

\begin{IEEEbiography}[{\includegraphics[width=1in,height=1.25in,clip,keepaspectratio]{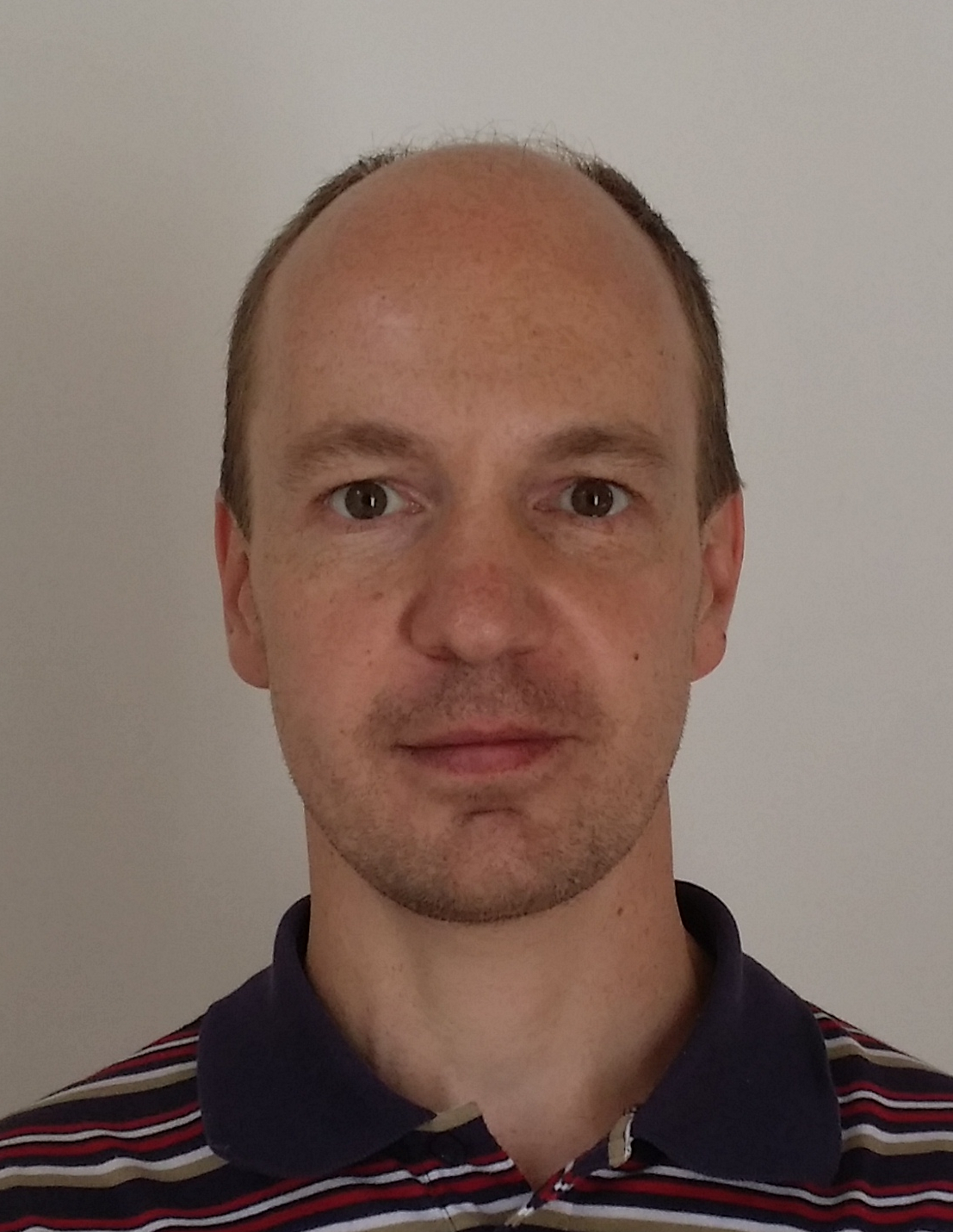}}]{Milan Hlad\'ik}
 is an associate professor at the Department of Applied Mathematics of the Faculty of Mathematics and Physics, Charles University, Prague, Czech Republic. His research interests include interval computation, numerical analysis, optimization and operations research. He is a member of the editorial board of four international journals, including European Journal of Operational Research.
\end{IEEEbiography}
\begin{IEEEbiography}[{\includegraphics[width=1in,height=1.25in,clip,keepaspectratio]{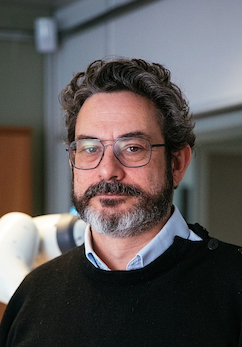}}]{David Daney} is a researcher at Inria Bordeaux Sud-Ouest, and leads the Project-Team Auctus, Inria-IMS (Univ. Bordeaux, Bordeaux INP, CNRS UMR 5218), Talence, France. He received the B. Eng in applied mathematics and computer science from University of Toulouse in 1996 and the M.Sc. and Ph.D. degrees in robotics from the University of Nice, Sophia Antipolis, France, in 1997 and 2000, respectively. He spent two years as a postdoctoral associate with C.M.W., France, to design a robotics system (2000) and with the LORIA Lab, Nancy, France, on computer arithmetic (2001). In 2002, he was an invited researcher with McGill University, Rutgers University, Laval University. From 2003 to 2013, he has been an Inria Research Scientist with the Inria Sophia Antipolis Research Center. His research interests are in the areas of parallel robotics, calibration, interval analysis, mechanical design, and assistive systems. Since 2013, he has been at the Inria research center in Bordeaux and in 2017 he founded the Auctus team in collaboration with the ENSC and the IMS. The aim of this team is to design collaborative robotic systems for man at work in industrial environments.
\end{IEEEbiography}

\end{document}